%% file: acl_latex.tex
\documentclass[11pt]{article}

\usepackage[preprint]{acl}

\usepackage{times}
\usepackage{latexsym}

\usepackage[T1]{fontenc}

\usepackage[utf8]{inputenc}
\usepackage{tcolorbox}
\usepackage{microtype}

\usepackage{inconsolata}

\usepackage{graphicx}
\usepackage{booktabs}
\usepackage{multirow}
\usepackage{graphicx}

\usepackage{amsmath}
\usepackage{amssymb}
\usepackage{mathtools}
\usepackage{amsthm}

\usepackage{multirow}
\usepackage[table,xcdraw]{xcolor}
\usepackage[normalem]{ulem}
\useunder{\uline}{\ul}{}
\usepackage[capitalize,noabbrev]{cleveref}
\tcbuselibrary{breakable}
\theoremstyle{plain}
\newtheorem{theorem}{Theorem}[section]

\theoremstyle{definition}

\newtheorem{assumption}[theorem]{Assumption}
\theoremstyle{remark}

\usepackage{hyperref}

\usepackage{algorithm}
\usepackage{algorithmic}
\usepackage{enumitem}

\title{Improve Large Language Model Systems with User Logs}

\usepackage{authblk}

\author[1,2]{\textbf{Changyue Wang}\thanks{cy-wang24@mails.tsinghua.edu.cn}}
\author[1,2]{\textbf{Weihang Su}}
\author[1]{\textbf{Qingyao Ai}\thanks{Corresponding Author: aiqy@tsinghua.edu.cn}}
\author[2]{\textbf{Xingzhao Yue}}
\author[2]{\protect\\\textbf{Rui Zhang}}
\author[2]{\textbf{Xiaojia Chang}}
\author[1]{\textbf{Yiqun Liu}}

\affil[1]{Department of Computer Science and Technology, Tsinghua University}
\affil[2]{TikTok}

\begin{document}
\maketitle
\begin{abstract}
Scaling training data and model parameters has long driven progress in large language models (LLMs), but this paradigm is increasingly constrained by the scarcity of high-quality data and diminishing returns from rising computational costs. As a result, recent work is increasing the focus on continual learning from real-world deployment, where user interaction logs provide a rich source of authentic human feedback and procedural knowledge.
However, learning from user logs is challenging due to their unstructured and noisy nature. Vanilla LLM systems often struggle to distinguish useful feedback signals from noisy user behavior, and the disparity between user log collection and model optimization (e.g., the off-policy optimization problem) further strengthens the problem.  To this end, we propose UNO (User log-driveN Optimization), a unified framework for improving LLM systems (LLMsys) with user logs. UNO first distills logs into semi-structured rules and preference pairs, then employs query-and-feedback-driven clustering to manage data heterogeneity, and finally quantifies the cognitive gap between the model’s prior knowledge and the log data. This assessment guides the LLMsys to adaptively filter out noisy feedback and construct different modules for primary and reflective experiences extracted from user logs, thereby improving future responses. Extensive experiments show that UNO achieves state-of-the-art effectiveness and efficiency, significantly outperforming Retrieval Augmented Generation (RAG) and memory-based baselines.~\footnote{We have open-sourced our code at https://github.com/bebr2/UNO .}
\end{abstract}

\section{Introduction}
In recent years, the development of large language models (LLMs) has largely followed scaling laws, in which scaling training data and model parameters yields predictable capability gains~\cite{kaplan2020scalinglawsneurallanguage,fang2024scaling}.  However, this progress is increasingly constrained by the scarcity of high-quality training data and diminishing marginal returns relative to rapidly rising computational costs~\cite{10.5555/3692070.3694094}. These challenges motivate a growing interest in moving beyond static, pre-trained models toward dynamic systems that continually adapt through real-world interaction~\cite{10.1145/3735633, lifelong}.

In information retrieval (IR), exploiting user logs (e.g., click behavior) to construct self-improving loops is a well-established paradigm for continuous optimization without human supervision~\cite{recommender, croft2010search}. Similarly, deployed LLMs accumulate extensive user logs  (e.g., prompts, model outputs, and user feedback) offering large-scale, authentic feedback and valuable procedural memories. These logs represent a rich but underexplored resource for optimizing LLM-based systems (LLMsys)~\cite{ai2025memorybenchbenchmarkmemorycontinual}.

Despite these successes in IR, log-driven continual learning for LLMsys faces fundamental challenges.
First, unlike fine-grained click behaviors in IR~\cite{searchengine, userfeedback}, user feedback for LLMs is sparse and unstructured, making it difficult to extract useful signals. For example, LLMsys provides responses in free-form text, and users either directly exit without interacting with it or only provide coarse-grained verbal or action feedback (e.g., copy) to the response as a whole. 
Our preliminary study reveals  LLMs struggle to distinguish useful knowledge and feedback from noisy signals in user logs, presenting a \textit{Signal-or-Noise Dilemma}.\footnote{We describe this in detail in Section~\ref{cha:cga}.}
Second, in contrast to IR systems, where knowledge is explicitly stored outside the retrieval model, LLMs entangle knowledge and reasoning implicitly within their parameters, making autonomous learning highly susceptible to catastrophic forgetting in the parameter space~\cite{gao2025surveyselfevolvingagentspath}.
Finally, logs collected from product systems may mismatch the target model's distribution. Such off-policy optimization is particularly challenging for LLMs due to their massive parameter size and black-box nature~\cite{levine2020offlinereinforcementlearningtutorial}.

To address these challenges, we introduce \textbf{UNO} (User log-driveN Optimization), a framework for continual LLM evolution from user logs. UNO has two stages: learning and inference.
In the learning stage, UNO constructs multiple experience modules from collected user logs, each implemented as a specialized adapter functioning either as a generation expert or a critic for iterative refinement.
Specifically, for each user session, UNO extracts a semi-structured rule set\footnote{Actionable editing guidelines distilled from unstructured logs. See Section~\ref{cha:rule} for details.} from the feedback signal  (e.g., user corrections), which tells the models how to improve the response. Following these rules, it generates a revised response, forming a preference pair.
Then it conducts agglomerative clustering~\cite{müllner2011modernhierarchicalagglomerativeclustering} based on queries and rule sets to divide the logs into multiple clusters. Inspired by Dewey’s philosophy on Experience and Nature~\cite{dewey2012experience}, UNO constructs two types of specialized parameter-efficient adapters (e.g., LoRA~\cite{lora}) for these clusters: \textit{primary experience modules} and \textit{reflective experience modules}. The primary experience module acts as an expert for direct generation. 
The reflective experience module learns a critic adapter that critiques a draft response and provides actionable revision suggestions, enabling an iterative refine-and-regenerate loop, without directly changing the base model's weights at inference time.
During inference stage, UNO retrieves experience modules relevant to the current context and composes them with the base LLM to generate enhanced responses.

To determine which type of experience module to use for each cluster, we define \textbf{cognitive gap} to measure how well the experience extracted from user logs matches the original LLM's understanding of the cluster's queries.
If the cognitive gap is small, we train a primary experience module and evaluate it (i.e., the LoRA learned directly from user logs) using a simulated verifier based on the LLM-as-Judge paradigm~\cite{li2025generationjudgmentopportunitieschallenges}, using the extracted rules as context. If the verifier determines the performance is satisfactory, the cluster remains a primary experience cluster.
Otherwise, or if the initial cognitive gap is large, the cluster is treated as a reflective experience cluster and user logs are used to construct the critic models that provide suggestions instead of directly changing the parameters.
Since user logs may be collected under historical policies or contain noises, this module-level adaptation with verification or critique mitigates off-policy risks and prevents harmful direct updates to the base model.
We evaluate UNO on MemoryBench~\cite{ai2025memorybenchbenchmarkmemorycontinual}, a comprehensive continual learning benchmark covering multiple datasets, domains, tasks, and languages, and further evaluate performance on wild and off-policy user logs (WildFB~\cite{peng2026wildrewardlearningrewardmodels}).
Compared to conventional training, Retrieval-Augmented-Generation (RAG)~\cite{lewis2020retrieval}, and memory-based methods, UNO demonstrates state-of-the-art performance in both effectiveness and efficiency.

To summarize, our contributions are as follows:
\begin{itemize}[leftmargin=*]
    \item We conduct, as far as we know, the first study on the optimization of LLMsys capabilities using user logs and characterize several challenges, including the \textit{Signal–or-Noise Dilemma}. 
    \item We propose UNO, a novel framework for LLMsys continual learning from user logs. It features cluster-based multiple experience management driven by model-log cognitive gaps.
    \item Extensive experiments demonstrate UNO achieves state-of-the-art performance across  tasks and languages, significantly outperforming RAG and memory-based baselines.
\end{itemize}

\section{Related Work}

\subsection{Memory for LLM Systems}
Research on LLM memory has evolved into complex management systems (e.g., MemoryOS~\cite{kang-etal-2025-memory}, Mem0~\cite{Mem0}) and corresponding benchmarks (e.g., LoCoMo~\cite{maharana-etal-2024-evaluating}). A key distinction between these works and ours is data granularity. While existing systems focus on user-level personalization evaluated on individual historical logs, we leverage system-wide logs. We aggregate diverse user feedback to drive continual learning and evaluate performance on unseen online requests. Furthermore, they focus on declarative memory while in our scenarios, procedural memory is more critical.

\subsection{Evolutionary Agent Frameworks}

Continual learning~\cite{shi2025continuallearning} often employs evolutionary frameworks~\cite{zhai2025agentevolverefficientselfevolvingagent, feng2025evoagentselfevolvingagentcontinual} (e.g., AlphaEvolve~\cite{novikov2025alphaevolvecodingagentscientific}) and memory integration (e.g., MemRL~\cite{zhang2026memrlselfevolvingagentsruntime}) for self-evolution. However, these systems rely heavily on explicit, repeated reward signals from environments (e.g., code executors). In contrast, our work optimizes LLMs using real-world user logs. This presents the unique challenge of driving continual learning from implicit, non-repeated feedback, as regular users can not act as dedicated annotators.

\subsection{User Log-driven Optimization Paradigms}

User log-based optimization is well-established in search and recommendation, but remains relatively nascent in the LLM domain. Recent works have explored using specific behavioral signals (e.g., clicks) for query suggestions~\cite{yin2025clickspreferencemultistagealignment} or action feedback for emotional support~\cite{han2025reinforcementlearninguserfeedback}, but they are limited to narrow task settings.
While datasets like WildChat~\cite{wildchat} and WildFeedback~\cite{wildfeedback} provide real-world logs for preference alignment, optimizing LLMsys directly on them is challenging. Their off-policy nature, originating from various, often closed-source models, makes it difficult to distinguish feedback-driven capability gains from mere knowledge distillation. Moreover, they lack downstream tasks to evaluate capability improvements. MemoryBench~\cite{ai2025memorybenchbenchmarkmemorycontinual} resolves these bottlenecks by offering system-specific dialogues and using future queries for evaluation. This enables our work to introduce a novel paradigm enhancing core LLMsys capabilities via raw user logs.

\section{Methodology}
In this section, we describe the UNO (\textbf{U}ser log-drive\textbf{N} \textbf{O}ptimization) framework in detail. UNO aims to enable autonomous, continual learning for LLMs by leveraging raw user logs. We first introduce the task and define the notations, then provide an overview of UNO, and describe the complete workflow (preprocessing, training, and inference).

\subsection{Problem Formulation and Preliminaries}
Let $\pi_{\theta}$ denote a base LLM policy with parameters $\theta$. User logs are represented as a set of sessions $\mathcal{D} = \{S_1, S_2, \dots, S_N\}$, where each session $S_i = (q_i, y_i, \mathcal{F}_i)$ comprises a user query $q_i$, an initial response $y_i \sim \pi_{\theta}(\cdot \mid q_i)$, and a subsequent interaction trajectory $\mathcal{F}_i = \{(u_{i,t}, r_{i,t})\}_{t=1}^{T_i}$ of $T_i \ge 0$ dialogue turns ($u_{i,t}$ is user input, $r_{i,t}$ is LLM response). In real user logs, such trajectories may reflect explicit corrections or implicit preferences from users, but may also contain irrelevant noise or be empty (i.e., $T_i = 0$).
The goal is to exploit $\mathcal{D}$ to optimize $\pi_{\theta}$ into an improved policy $\pi'$. 
Successful optimization requires that for a new test query $q_{test}$, the expected response quality $\mathcal{V}$ improves:

{\footnotesize
\[
\mathbb{E}_{q_{test}}[\mathcal{V}(\pi' (\cdot \mid q_{test}))] \geq
\mathbb{E}_{q_{test}}[\mathcal{V}( \pi_{\theta}(\cdot \mid q_{test}))],
\]
}
Note that user logs may contain either high-value signals or noise. The optimization framework requires two core capabilities: 1) \textbf{Positive Adaptivity} to convert valuable signals into capability improvements; and (2) \textbf{Noise Robustness} to maintain stable performance against low-quality logs. 

Although not originally designed for this task, several existing approaches can be adapted.
Traditional full fine-tuning updates model parameters to $\theta'$ to directly fit the log distribution, yielding $\pi_{FT}(\cdot \mid q_i) = \pi_{\theta'}(\cdot \mid q_i)$. Memory-based systems retrieve external memory entries to construct a context $M_i$, yielding $\pi_{Mem}(\cdot \mid q_i) = \pi_{\theta}(\cdot \mid q_i, M_i)$.

\subsection{Overview of UNO}

\begin{figure*}
    \centering
    \includegraphics[width=0.92\linewidth]{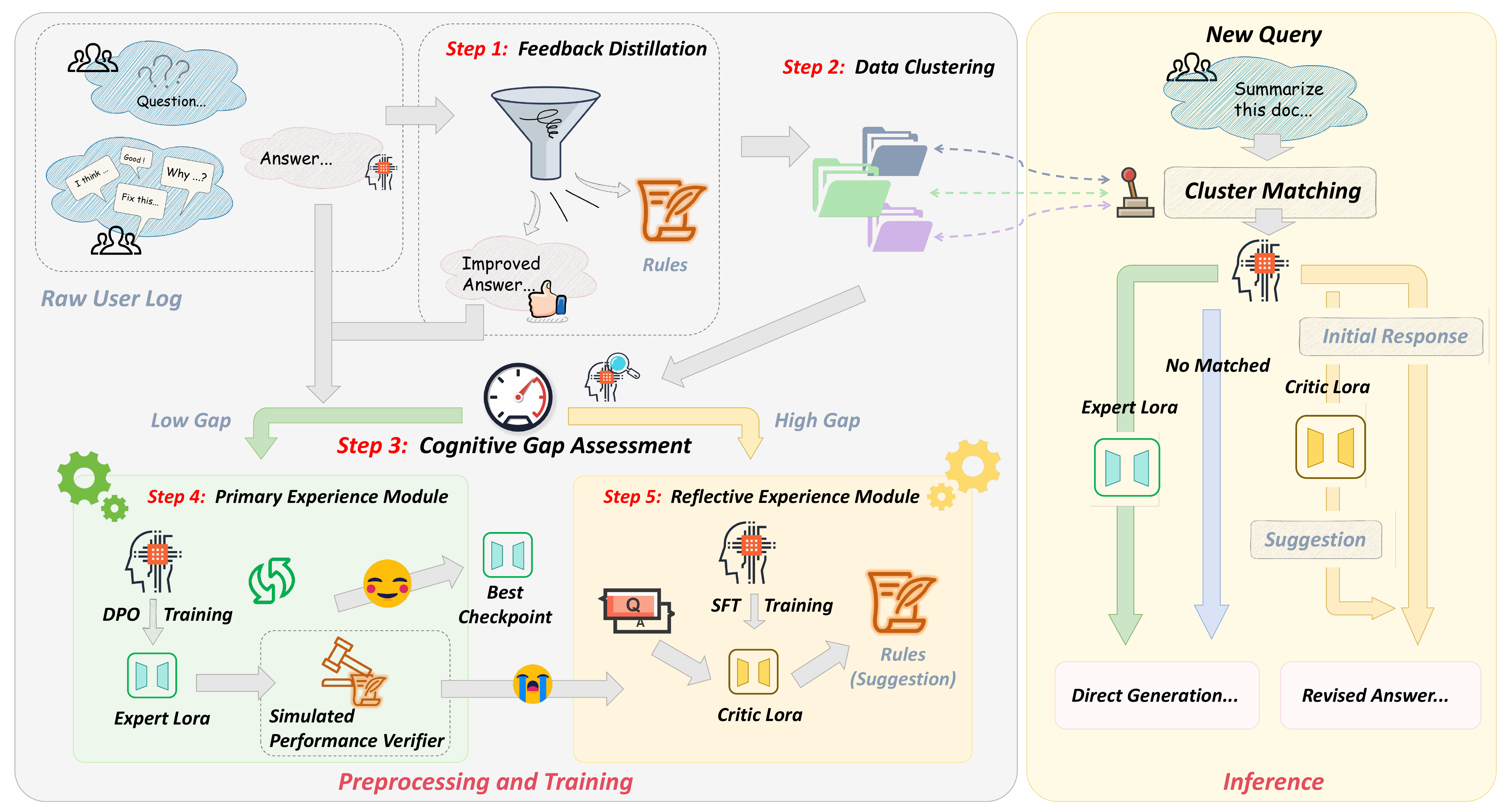}
    \caption{The workflow of UNO. UNO first distills and filters raw user logs, then performs clustering and a cognitive gap assessment to select the type of experience module (primary or reflective). At inference time, UNO identifies the appropriate cluster and applies an inference strategy aligned with the type of that cluster.}
    \label{fig:workflow}
\end{figure*}
UNO is the first optimization framework enabling LLMsys to continuously and adaptively evolve on user logs while maintaining noise robustness.  UNO consists of four stages: 1) \textbf{Preprocessing}: it filters and distills noisy user logs into a preference dataset, partitions the data, and assesses how well the LLM accepts user feedback, which is used to estimate optimization difficulty. Guided by this, the framework constructs either a \textbf{Primary Experience Module} or a \textbf{Reflective Experience Module} for different clusters,  inspired by John Dewey’s philosophy on Experience and Nature~\cite{dewey2012experience}. The former trains an Expert LoRA to directly answer queries, while the latter trains a Critic LoRA that offers guidance on the base LLM’s initial response during inference. For the \textbf{Primary Experience Module}, we further build a simulated performance verifier using distilled rules to evaluate trained LoRAs before deployment. Finally, at 4) \textbf{Inference Workflow}, for a new query, UNO first matches it to the nearest cluster. If the cluster has a primary experience module, the corresponding Expert LoRA is loaded to answer. If it is based on reflective experience, the base LLM generates an initial answer, the Critic LoRA provides feedback, and the LLM revises it accordingly. Figure~\ref{fig:workflow} and Algorithm~\ref{alg} outline the complete workflow and training pseudocode, and Appendix~\ref{apd:casestudy} shows UNO case studies.

\subsection{Preprocessing}

\subsubsection{Raw User Log Filtering and Distillation}
\label{cha:rule}
Raw user logs consist of unstructured text and are potentially noisy~\cite{ai2025memorybenchbenchmarkmemorycontinual}. We utilize the LLM to filter meaningless content and convert the rest into semi-structured rule sets.
Specifically, for each dialogue session $(q_i, y_i, \mathcal{F}_i)$, we discard samples with empty $\mathcal{F}_i$. The base LLM then performs \textit{feedback filtering and distillation}, transforming unstructured feedback into an explicit rule set $\mathcal{R}_i = \{r_{i,1}, r_{i,2}, \dots, r_{i,k}\}$. 
For instance, if a user complains that a generated journalistic report of an academic paper is too technical and lacks social impact, the distilled rules would be: \textit{1) Use accessible, non-technical news style; 2) Emphasize real-world impact.} 
We remove items with empty rule sets, as these correspond to uninformative dialogues.
Subsequently, the LLM is guided to revise the original response $y_i$ under $\mathcal{R}_i$'s  constraints, yielding a better response $y_i^w \sim \pi_{\theta}(\cdot \mid q_i, \mathcal{R}_i, y_i)$. This constructs a preference pair $(q_i, y_i^w, y_i^l)$, where $y_i^l = y_i$.

\subsubsection{Dual-feature Data Clustering}  
\label{cha:clustering}
To facilitate training and more precise assessment of noise risk, we cluster the data by jointly leveraging information from both the query and the rule set. Specifically, we construct a vector $v_i = \text{Norm}[E(q_i) \oplus E(\mathcal{R}_i)]$ where $E(\cdot)$ represents a text encoder and $\oplus$ denotes concatenation.
We employ the standard hierarchical agglomerative clustering implementation from \textit{scikit-learn}~\cite{scikit-learn}, using Ward's linkage~\cite{müllner2011modernhierarchicalagglomerativeclustering} to iteratively merge cluster pairs minimizing the increase in intra-cluster variance, ultimately yielding a set of clusters $\mathcal{C}$. This aligns samples within each cluster in semantic intent and applicable rules, thereby lowering the complexity of model training. We further present the clustering details at Appendix~\ref{apd:clustering}.

\subsubsection{Cognitive Gap Assessment}  
\label{cha:cga}
\paragraph{\textbf{Signal-or-Noise Dilemma.}} 
The main challenge in exploiting user logs for model optimization is determining whether they truly contribute to improvement.
While LLMs filter noise during distillation, this self-filtering fails when incoming data exceeds their capability boundaries. 
When signals in logs differ substantially from the model’s cognition, two interpretations emerge: (1) they encode unmastered but valuable hard knowledge; or (2) they are dominated by harmful noise or user bias. A self-improving system cannot theoretically distinguish these cases, creating the \textbf{\textit{Signal-or-Noise Dilemma}}.  
To address this, we introduce the \textbf{\emph{cognitive gap}} to reflect task difficulty and identify potentially high-risk noise within clusters.

\textbf{Assessment Process.}  
We quantify the cognitive gap through three steps:  \textbf{1) \textit{Rule Prediction}}: For each query $q_i$ in cluster $\mathcal{C}_k$, the base model $\pi_\theta$ independently generates the rule set $\mathcal{R}_i^{LLM}$ without access to user logs.  
\textbf{2) \textit{Gap Quantification}}: A pre-trained reranker computes the semantic distance between predicted and distilled rules, yielding a cognitive gap score $g_i = \text{Dist}(\mathcal{R}_i^{LLM}, \mathcal{R}_i) \in [0, 1]$.  
\textbf{3) \textit{Cluster-level Profiling}}: We calculate the average gap $\mu_k = \frac{1}{|\mathcal{C}_k|} \sum_{i \in \mathcal{C}_k} g_i$ for each cluster. Based on the value of $\mu_k$ and a predefined threshold $\tau^*$, clusters are classified as ``Low-Gap'' or ``High-Gap''.

\textbf{Theoretical Framework of Cognitive Gap.}  
We formalize the cognitive gap to  evaluate optimization risk. Let $H$ denote high-quality data and $N$ noisy data, with priors $P(H)=\alpha$ and $P(N) = 1 - \alpha$. The core intuition lies in the LLM's noise discrimination capability across cognitive regions. In low-gap regions, high-quality data represent minor refinements of the model’s cognition and are largely aligned with it, whereas noisy data are inherently random and lack this alignment, making the two highly separable. In contrast, high-gap regions contain samples deviating substantially from the model’s current capabilities. Here, an LLM cannot reliably distinguish ``informative novel signals'' from ``uninformative random noise,''  causing their cognitive distributions to converge. We further discuss this at Appendix~\ref{apd:assumdiscuss}, and formalize it with the following assumption: 

\begin{assumption}
\label{assum:gap_dist}
For a small threshold $\tau > 0$, in the low-gap region, $P_{N}(g_i \leq \tau) \ll P_H(g_i \leq \tau)$. Conversely, in the high-gap region, the distributions of high-quality and noisy data converge, such that $P_H(g_i > 1 - \tau) \approx P_{N}(g_i > 1 - \tau)$.
\end{assumption}

\begin{theorem}[Noise Risk Bound]
For data with small cognitive gaps, the posterior probability of noise $P(N|g_i \leq \tau)$ is strictly bounded:

{\small
\begin{equation}
P(N|g_i \leq \tau) \leq \frac{1}{1 + \alpha \frac{P_H(g_i \leq \tau)}{P_{N}(g_i \leq \tau)}}
\end{equation}
}
\label{theo1}
\end{theorem}


The proof is based on Bayes' theorem and detailed in Appendix~\ref{apd:proof32}. As $P_H/P_{N}$ grows sharply in the low-gap region, training risk decreases rapidly with smaller $\tau$. In contrast, for high-gap regions ($g_i > 1 - \tau$), the posterior approaches $P(N|g_i > 1-\tau) \approx 1-\alpha$, indicating hard samples become indistinguishable from noise without a meaningful risk bound. Next, we illustrate that under some assumptions (e.g., $L_\Psi$-Lipschitz continuity), clustering concentrates the cognitive gaps within each cluster, forming distinct low- and high-gap clusters to guide our optimization strategy.


\begin{theorem}[Variance Reduction via Clustering]
Under UNO’s dual-feature clustering, the intra-cluster variance of the cognitive gap $\text{Var}(g_i)$ is bounded by the cluster diameter $\epsilon$:

{\small
\begin{equation}
\begin{scriptsize}
Var(g_i) < [C \cdot (1 + L_\Psi) \cdot \epsilon]^2
\end{scriptsize}
\end{equation}
}

\noindent where $C$ and $L_\Psi$ are constants.
\label{theo2}
\end{theorem}

The proof bounds intra-cluster gap differences using Lipschitz continuity and the triangle inequality (Appendix~\ref{apd:proof33}).  According to Theorem~\ref{theo2}, beyond task decoupling, clustering reduces the variance of $g_i$ by bounding diameter $\epsilon$. This enables reliable decisions on whether a cluster is suitable for direct fine-tuning based on the mean gap $\mu_k$, since the estimation of the expectation becomes more accurate with lower sample variance. Clusters with low $\mu_k$ reliably contain absorbable knowledge for the model. Conversely, clusters with high $\mu_k$ lack reliable risk bounds and are likely noisy. In such cases, UNO adopts \textit{Reflective Experience Construction} to maintain base model stability.

\subsection{Primary Experience Module}
For clusters below the cognitive gap threshold $\tau^*$, the data serve only as a fine-tuning of the model’s cognition, and the risk of noise is controlled.  Thus, we directly internalize log-derived capability by training a LoRA on preprocessed preference data.

\textbf{Cluster-specific Preference Learning.}  
We partition data into training and validation sets. For low- or moderate-gap clusters, we train an \textit{Expert LoRA} $\Delta\theta_C$ using the Direct Preference Optimization (DPO~\cite{dpo}) loss alongside a negative log-likelihood (NLL) loss on chosen responses for stability. This directly embeds new capabilities into the parameter space.

\textbf{Simulation-based Validation.}  
We construct a simulated performance verifier from the validation set to assess training efficacy. For each query $q_i$, using $\mathcal{R}_i$ as context, an LLM-as-Judge compares $y_{i}^{C}\sim \pi_{\theta + \Delta\theta_C(\cdot | q_i)}$ against the base $y_i \sim \pi_{\theta}(\cdot | q_i)$ per epoch.
Crucially, by strictly grounding the judge in the rule set $\mathcal{R}_i$ extracted from actual user feedback, we effectively mitigate potential "judge-hacking" where the adapter might otherwise exploit the base model's inherent biases.
We retain the expert module if its peak win rate $\text{WinRate}_{best} > \gamma$. Otherwise, we fall back to \textit{Reflective Experience Construction}.
Furthermore, to accelerate the evaluation process and further alleviate potential reward-hacking, we pre-filter responses that are essentially empty or deviate drastically from the original response (e.g., those with a BLEU-4 score~\cite{papineni-etal-2002-bleu} below a small threshold, such as 0.05).

\subsection{Reflective Experience Module}

For clusters with large cognitive gaps or those failing primary validation, we train a \textit{Critic LoRA} $\Delta \omega_{C}$, which is designed to provide ``pseudo user feedback'' $\hat{\mathcal{R}}$ for the model’s initial responses.
The training input consists of a query $q_i$ and a model response $y_i$, while the output is an extracted rule set $\mathcal{R}_i$, optimized using a negative log-likelihood loss. The checkpoint with the lowest validation loss becomes the final Critic LoRA.
During inference, Critic LoRA critiques the base LLM’s initial output, allowing the base LLM to refine its response. Our experiments show that, for clusters with high noise risk, this strategy effectively leverages informative signals in the logs and significantly reduces noise-induced collapse compared to directly training an Expert LoRA.
Moreover, the Reflective Experience Module prevents UNO’s optimization framework from becoming overly conservative, ensuring robust utilization of valuable user-log signals even for clusters filtered out by cognitive gap assessment or unsuccessful primary experience construction.

\subsection{Inference Workflow}

During inference, for a query $q_{new}$, the policy $\pi_{\theta'}$ computes the euclidean distance between $q_{new}$'s embedding and the query embeddings of all cluster centroids, and assigns it to the closest cluster $C^*$. Response generation depends on $C^*$'s type:
\begin{itemize}[leftmargin=*]
    \item \textbf{Primary Path}: For a \textbf{primary experience cluster}, we sample directly from its \textit{Expert LoRA}:
    $y \sim \pi_{\theta + \Delta\theta_{C^*}}(\cdot | q_{new})$.
    \item \textbf{Reflective Path}: For a \textbf{reflective experience cluster}, we use a two-stage generation:
    $y_{\text{init}} \sim \pi_{\theta}(\cdot | q_{new})$, $\hat{\mathcal{R}} \sim \pi_{\theta + \Delta \omega_{C^*}}(\cdot | q_{new}, y_{\text{init}})$, and
    $y \sim \pi_{\theta}(\cdot | q_{new}, y_{\text{init}}, \hat{\mathcal{R}})$.
    \item \textbf{Fallback}: If $q_{new}$ is an outlier, i.e., having a distance to the nearest centroid exceeding threshold $d$, we revert to the base policy: 
    $y \sim \pi_{\theta}(\cdot | q_{new})$.
\end{itemize}

\section{Experimental Setup}
In this section, we introduce the evaluation and implementation details. Further experimental details are provided in Appendix~\ref{apd:detailsexp}.

\subsection{Evaluation}
We evaluate the methods on \textbf{MemoryBench}~\cite{ai2025memorybenchbenchmarkmemorycontinual}, the first benchmark for user log-driven continual learning, where a carefully validated \textit{User Simulator} produces interaction logs conditioned on each model’s responses. The benchmark includes user logs collected during the early stage (the training set) and new requests (the test set). All strategies are performed on the training set and evaluated on the test set across four task-based datasets: Short-input–Long-output, Short-input–Short-output, Long-input–Long-output, and Long-input–Short-output. 
Each task contains logs (dialogue sessions, including model responses and user verbal feedback) derived from queries in distinct data subsets. Since test queries span datasets, MemoryBench first computes dataset-specific metrics and applies min–max normalization or z-score for final task-level performance. For both, higher values correspond to better performance. Raw dataset-specific results are in Appendix~\ref{apd:finegrained}.

Moreover, to assess performance on real human interactions, we experiment on the \textbf{WildFB}~\cite{peng2026wildrewardlearningrewardmodels}, a subset of WildChat. We clean its test set and partition it into logs for evolution and evaluation tasks (detailed in Appendix~\ref{apd:evaluation_details}). Using WildReward-8B from the original paper as the evaluator, we compute the win rate of generated versus original responses. This task mainly evaluates optimization frameworks under realistic noise and user behavior. Notably, a gap exists between this setting and our primary evaluation, as the models in dialogues differ from those being optimized.

We evaluate strategies using Qwen3-8B~\cite{yang2025qwen3technicalreport} (no-thinking mode) and phi-4 (14B)~\cite{abdin2024phi4technicalreport}. We also test a lighter variant, UNO-Single, in which the Reflective Path is removed, and responses for filtered clusters are generated directly by the base LLM. Thus, UNO-Single adopts more conservative optimization without adding additional inference tokens. Significance is tested via paired $t$-test~\cite{student1908probable}.

\subsection{Baselines}Prior work does not propose a pipeline specifically tailored for user log-driven optimization. In addition to the direct output of the base model, we compare against three related method categories:
1) \textbf{Retrieval-Augmented Generation (RAG)}: We use either BM25~\cite{robertson2009probabilistic} or Qwen3-Embedding-0.6B~\cite{zhang2025qwen3embeddingadvancingtext} as the retriever. 
2) \textbf{Memory for LLM Systems}: We evaluate MemOS~\cite{li2025memosmemoryosai}, ReMem~\cite{wei2025evomemorybenchmarkingllmagent}, A-Mem~\cite{xu2025amem}, Mem0~\cite{Mem0}, and MemoryOS~\cite{kang-etal-2025-memory}.
3) \textbf{Training Methods}: Using preference data constructed by UNO, we evaluate two offline training approaches: supervised fine-tuning (SFT) and direct preference optimization (DPO).


\subsection{Implementation Details of UNO}
We implement UNO using Qwen3-Embedding-0.6B as the encoding model for clustering. To fully exploit both the information contained in the original logs and the high-value elements of the extracted rules when constructing preference data, each of the two is used as contextual input to generate candidate outputs, from which the superior one is selected. For cognitive gap assessment, we use Qwen3-Reranker-0.6B~\cite{zhang2025qwen3embeddingadvancingtext} to evaluate whether each extracted rule is independent of the base LLM's own knowledge, extracting the model-generated scores to compute $\text{Dist}(\mathcal{R}_i, \mathcal{R}_i^{\text{LLM}})$. In the simulated performance verifier, we evaluate both each checkpoint and the base LLM using an LLM-as-Judge paradigm, where the judge model (the base LLM) provides evaluation rationales and scores.

\textbf{Online Evolution Settings.} We further evaluate UNO on phi-4 model in an online evolution setting to test whether it can continue learning using newly collected logs after deployment. We evenly split the training data into two independent batches. After optimizing the model on the first batch, we allow UNO to interact with the User Simulator using queries from the second batch to collect new user logs. A key consideration in online evolution is whether clustering should preserve the original centroids. We perform clustering over all user logs: if the number of clusters changes, we adopt the new clustering results; otherwise, we keep the original centroids and perform continual training or retraining based on specific performance criteria.
\label{cha:online}

\begin{table*}[!t]
\centering
\caption{Main Results in MemoryBench. "Short-Long" denotes the "Short-Input-Long-Output" task. Performance is measured using MemoryBench’s aggregate metrics ($\uparrow$): Norm-Score (min–max normalized) and Z-Score. Scores are only meaningful for relative comparison within the same model and dataset. We bold the best, underline the second best, and shade our proposed methods in gray. "Base" denotes the LLM’s initial response. ``*'' / ``**'' denotes significantly worse performance than UNO ($p<0.1$ or $p<0.05$). ``-'' means failure to generate valid output.}
\renewcommand{\arraystretch}{0.8}
\resizebox{0.95\textwidth}{!}{
\begin{tabular}{@{}ccccccccccc@{}}
\cmidrule[\heavyrulewidth](l{-6pt}r{-6pt}){1-11}
 &
   &
   &
  \multicolumn{2}{c}{\textbf{Short-Long}} &
  \multicolumn{2}{c}{\textbf{Short-Short}} &
  \multicolumn{2}{c}{\textbf{Long-Long}} &
  \multicolumn{2}{c}{\textbf{Long-Short}} \\ \cmidrule(l){4-11} 
\multirow{-2}{*}{\textbf{Model}} &
  \multirow{-2}{*}{\textbf{Type}} &
  \multirow{-2}{*}{\textbf{Method Name}} &
  \multicolumn{1}{c|}{\textbf{Norm-Score}} &
  \multicolumn{1}{c|}{\textbf{Z-Score}} &
  \multicolumn{1}{c|}{\textbf{Norm-Score}} &
  \multicolumn{1}{c|}{\textbf{Z-Score}} &
  \multicolumn{1}{c|}{\textbf{Norm-Score}} &
  \multicolumn{1}{c|}{\textbf{Z-Score}} &
  \multicolumn{1}{c|}{\textbf{Norm-Score}} &
  \textbf{Z-Score} \\ \midrule
\multicolumn{1}{c|}{} &
  \multicolumn{1}{c|}{\textbf{-}} &
  \multicolumn{1}{c|}{\textbf{Base}} &
  74.43** &
  -9.52** &
  72.35** &
  -2.31** &
  63.41\phantom{**} &
  14.87\phantom{**} &
  46.94\phantom{**} &
  6.02\phantom{**} \\ \cmidrule(l){2-11} 
\multicolumn{1}{c|}{} &
  \multicolumn{1}{c|}{} &
  \multicolumn{1}{c|}{\textbf{Embedding}} &
  73.22** &
  -17.04** &
  72.96** &
  4.21** &
  62.04*\phantom{*} &
  7.58** &
  48.41\phantom{**} &
  12.29\phantom{**} \\ \cmidrule(lr){3-3}
\multicolumn{1}{c|}{} &
  \multicolumn{1}{c|}{\multirow{-2}{*}{\textbf{RAG}}} &
  \multicolumn{1}{c|}{\textbf{BM25}} &
  74.54** &
  -8.83** &
  72.71** &
  3.50** &
  61.49** &
  6.50** &
  48.40\phantom{**} &
  10.87\phantom{**} \\ \cmidrule(l){2-11} 
\multicolumn{1}{c|}{} &
  \multicolumn{1}{c|}{} &
  \multicolumn{1}{c|}{\textbf{MemOS}} &
  72.29** &
  -23.07** &
  66.57** &
  -27.26** &
  57.95** &
  -7.71** &
  47.14\phantom{**} &
  -12.16** \\ \cmidrule(lr){3-3}
\multicolumn{1}{c|}{} &
  \multicolumn{1}{c|}{} &
  \multicolumn{1}{c|}{\textbf{ReMem}} &
  71.95** &
  -25.55** &
  70.80** &
  -7.69** &
  59.09** &
  -4.30** &
  43.78** &
  -7.70** \\ \cmidrule(lr){3-3}
\multicolumn{1}{c|}{} &
  \multicolumn{1}{c|}{} &
  \multicolumn{1}{c|}{\textbf{A-Mem}} &
  70.36** &
  -35.02** &
  70.15** &
  -10.89** &
  60.55** &
  2.22** &
  47.94\phantom{**} &
  5.91\phantom{**} \\ \cmidrule(lr){3-3}
\multicolumn{1}{c|}{} &
  \multicolumn{1}{c|}{} &
  \multicolumn{1}{c|}{\textbf{Mem0}} &
  68.25** &
  -51.93** &
  61.07** &
  -48.34** &
  58.29** &
  -5.70** &
  -\phantom{**} &
  -\phantom{**} \\ \cmidrule(lr){3-3}
\multicolumn{1}{c|}{} &
  \multicolumn{1}{c|}{\multirow{-6}{*}{\textbf{Memory}}} &
  \multicolumn{1}{c|}{\textbf{MemoryOS}} &
  74.62** &
  -8.96** &
  70.56** &
  -5.68** &
  45.96** &
  -59.89** &
  38.51** &
  -24.87** \\ \cmidrule(l){2-11} 
\multicolumn{1}{c|}{} &
  \multicolumn{1}{c|}{} &
  \multicolumn{1}{c|}{{\textbf{$\text{DPO}_\mathbf{\textit{w/ UNO Data}}$}}} &
  70.39** &
  -34.68** &
  72.89** &
  -8.77** &
  61.65** &
  6.87** &
  47.14\phantom{**} &
  2.30\phantom{**} \\ \cmidrule(lr){3-3}
\multicolumn{1}{c|}{} &
  \multicolumn{1}{c|}{} &
  \multicolumn{1}{c|}{\textbf{$\text{SFT}_\mathbf{\textit{w/ UNO Data}}$}} &
  69.95** &
  -37.85** &
  75.03\phantom{**} &
  5.42** &
  62.24*\phantom{*} &
  8.98** &
  {\ul 48.80}\phantom{**} &
  {\ul 7.00}\phantom{**} \\ \cmidrule(lr){3-3}
\multicolumn{1}{c|}{} &
  \multicolumn{1}{c|}{} &
  \multicolumn{1}{c|}{\cellcolor[HTML]{C0C0C0}\textbf{UNO-Single}} &
  \cellcolor[HTML]{C0C0C0}{\ul 76.36}** &
  \cellcolor[HTML]{C0C0C0}{\ul 2.99}*\phantom{*} &
  \cellcolor[HTML]{C0C0C0}{\ul 75.82}\phantom{**} &
  \cellcolor[HTML]{C0C0C0}{\ul 15.89}\phantom{**} &
  \cellcolor[HTML]{C0C0C0}{\ul 63.77}\phantom{**} &
  \cellcolor[HTML]{C0C0C0}{\ul 15.66}\phantom{**} &
  \cellcolor[HTML]{C0C0C0}46.81*\phantom{*} &
  \cellcolor[HTML]{C0C0C0}5.34\phantom{**} \\ \cmidrule(lr){3-3}
\multicolumn{1}{c|}{\multirow{-15}{*}{\rotatebox{90}{\textbf{Qwen3-8B}}}} &
  \multicolumn{1}{c|}{\multirow{-5}{*}{\textbf{Training}}} &
  \multicolumn{1}{c|}{\cellcolor[HTML]{C0C0C0}\textbf{UNO}} &
  \cellcolor[HTML]{C0C0C0}\textbf{77.09}\phantom{**} &
  \cellcolor[HTML]{C0C0C0}\textbf{7.16}\phantom{**} &
  \cellcolor[HTML]{C0C0C0}\textbf{76.26}\phantom{**} &
  \cellcolor[HTML]{C0C0C0}\textbf{21.54}\phantom{**} &
  \cellcolor[HTML]{C0C0C0}\textbf{64.23}\phantom{**} &
  \cellcolor[HTML]{C0C0C0}\textbf{17.74}\phantom{**} &
  \cellcolor[HTML]{C0C0C0}\textbf{49.99}\phantom{**} &
  \cellcolor[HTML]{C0C0C0}\textbf{13.49}\phantom{**} \\ \cmidrule[\lightrulewidth](l{-6pt}r{-6pt}){1-11}
\multicolumn{1}{c|}{} &
  \multicolumn{1}{c|}{\textbf{-}} &
  \multicolumn{1}{c|}{\textbf{Base}} &
  66.67** &
  -63.21** &
  72.17\phantom{**} &
  -21.81\phantom{**} &
  57.07\phantom{**} &
  -17.36\phantom{**} &
  46.46** &
  -1.59** \\ \cmidrule(l){2-11} 
\multicolumn{1}{c|}{} &
  \multicolumn{1}{c|}{} &
  \multicolumn{1}{c|}{\textbf{Embedding}} &
  66.27** &
  -64.57** &
  72.82\phantom{**} &
  -20.79\phantom{**} &
  54.43** &
  -28.63** &
  42.91** &
  -9.62** \\ \cmidrule(lr){3-3}
\multicolumn{1}{c|}{} &
  \multicolumn{1}{c|}{\multirow{-2}{*}{\textbf{RAG}}} &
  \multicolumn{1}{c|}{\textbf{BM25}} &
  69.33\phantom{**} &
  -44.72\phantom{**} &
  71.03\phantom{**} &
  -30.36** &
  53.67** &
  -31.60** &
  43.26** &
  -15.07** \\ \cmidrule(l){2-11} 
\multicolumn{1}{c|}{} &
  \multicolumn{1}{c|}{} &
  \multicolumn{1}{c|}{\textbf{MemOS}} &
  67.96** &
  -53.78** &
  67.40** &
  -40.49** &
  53.71** &
  -31.43** &
  42.83** &
  -26.43** \\ \cmidrule(lr){3-3}
\multicolumn{1}{c|}{} &
  \multicolumn{1}{c|}{} &
  \multicolumn{1}{c|}{\textbf{ReMem}} &
  66.66** &
  -61.81** &
  58.51** &
  -92.61** &
  47.72** &
  -54.90** &
  38.78** &
  -35.81** \\ \cmidrule(lr){3-3}
\multicolumn{1}{c|}{} &
  \multicolumn{1}{c|}{} &
  \multicolumn{1}{c|}{\textbf{A-Mem}} &
  67.27** &
  -57.95** &
  70.40** &
  -29.58\phantom{**} &
  53.22** &
  -32.94** &
  39.69** &
  -34.26** \\ \cmidrule(lr){3-3}
\multicolumn{1}{c|}{} &
  \multicolumn{1}{c|}{} &
  \multicolumn{1}{c|}{\textbf{Mem0}} &
  68.81*\phantom{*} &
  -48.25*\phantom{*} &
  65.11** &
  -58.86** &
  52.23** &
  -37.53** &
  -\phantom{**} &
  - \phantom{**}\\ \cmidrule(lr){3-3}
\multicolumn{1}{c|}{} &
  \multicolumn{1}{c|}{\multirow{-6}{*}{\textbf{Memory}}} &
  \multicolumn{1}{c|}{\textbf{MemoryOS}} &
  22.44** &
  -350.94** &
  49.72** &
  -256.49** &
  24.23** &
  -161.25** &
  25.59** &
  -89.25** \\ \cmidrule(l){2-11} 
\multicolumn{1}{c|}{} &
  \multicolumn{1}{c|}{} &
  \multicolumn{1}{c|}{\textbf{$\text{DPO}_\mathbf{\textit{w/ UNO Data}}$}} &
  67.17** &
  -58.71** &
  71.85\phantom{**} &
  {\ul -16.77}\phantom{**} &
  52.54** &
  -36.95** &
  32.59** &
  -54.20** \\ \cmidrule(lr){3-3}
\multicolumn{1}{c|}{} &
  \multicolumn{1}{c|}{} &
  \multicolumn{1}{c|}{\textbf{$\text{SFT}_\mathbf{\textit{w/ UNO Data}}$}} &
  66.97** &
  -60.05** &
  70.43** &
  -34.00** &
  55.63** &
  -24.44** &
  38.97** &
  -30.81** \\ \cmidrule(lr){3-3}
\multicolumn{1}{c|}{} &
  \multicolumn{1}{c|}{} &
  \multicolumn{1}{c|}{\cellcolor[HTML]{C0C0C0}\textbf{UNO-Single}} &
  \cellcolor[HTML]{C0C0C0}{\ul 69.61}** &
  \cellcolor[HTML]{C0C0C0}{\ul -43.08}** &
  \cellcolor[HTML]{C0C0C0}{\ul 72.91}\phantom{**} &
  \cellcolor[HTML]{C0C0C0}-18.04\phantom{**} &
  \cellcolor[HTML]{C0C0C0}{\ul 57.14}\phantom{**} &
  \cellcolor[HTML]{C0C0C0}{\ul -17.10}\phantom{**} &
  \cellcolor[HTML]{C0C0C0}{\ul 47.16}** &
  \cellcolor[HTML]{C0C0C0}{\ul 1.93}** \\ \cmidrule(lr){3-3}
\multicolumn{1}{c|}{\multirow{-15}{*}{\rotatebox{90}{\textbf{phi-4}}}} &
  \multicolumn{1}{c|}{\multirow{-5}{*}{\textbf{Training}}} &
  \multicolumn{1}{c|}{\cellcolor[HTML]{C0C0C0}\textbf{UNO}} &
  \cellcolor[HTML]{C0C0C0}\textbf{70.66}\phantom{**} &
  \cellcolor[HTML]{C0C0C0}\textbf{-36.69}\phantom{**} &
  \cellcolor[HTML]{C0C0C0}\textbf{73.19}\phantom{**} &
  \cellcolor[HTML]{C0C0C0}\textbf{-15.94}\phantom{**} &
  \cellcolor[HTML]{C0C0C0}\textbf{57.84}\phantom{**} &
  \cellcolor[HTML]{C0C0C0}\textbf{-13.62}\phantom{**} &
  \cellcolor[HTML]{C0C0C0}\textbf{52.60}\phantom{**} &
  \cellcolor[HTML]{C0C0C0}\textbf{18.40}\phantom{**} \\ \cmidrule[\heavyrulewidth](l{-6pt}r{-6pt}){1-11}
\end{tabular}
}
\label{tab:main}
\end{table*}

\begin{figure*}[!t]
    \centering
    \includegraphics[width=0.95\linewidth]{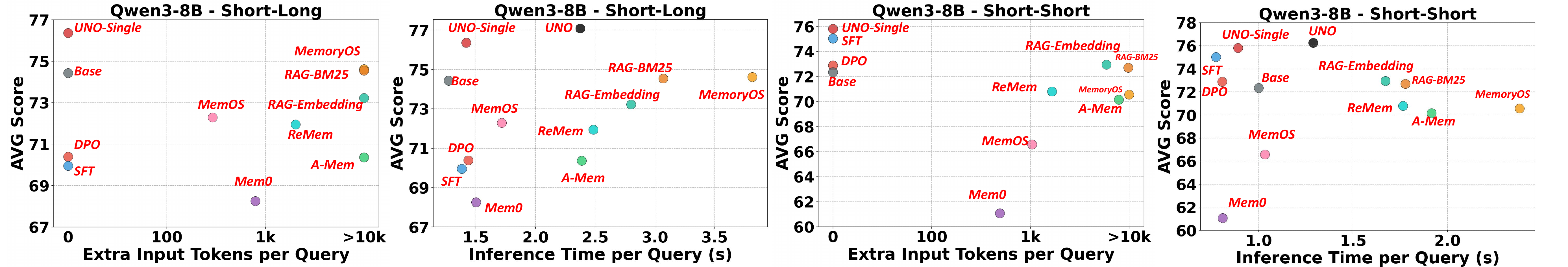}
    \caption{Performance (Norm-Score) vs. extra input tokens and inference time. Better case lie toward the upper-left. We report UNO-Single for the token metric, as full UNO's critique-and-revise step (multiple LLM calls/output tokens) is not captured here. However, the inference time evaluation includes the full UNO pipeline.}
    \label{fig:token}
\end{figure*}


\section{Experimental Results}
\subsection{Main Results}

\textbf{Main Performance.} Table~\ref{tab:main} reports the comprehensive results on MemoryBench. Across both Qwen3-8B and phi-4, as well as all four subtask datasets, UNO consistently outperforms all baseline methods. Compared with RAG and Memory systems, UNO adopts an adaptive dual-path framework, which enables more effective extraction and internalization of useful knowledge from noisy user logs.
Importantly, directly applying full DPO or SFT on the preference data constructed by UNO leads to unstable outcomes. Even after filtering via LLM-based evaluation, the model’s limited capacity makes it difficult to distinguish high-difficulty cases from noisy logs.
Figure~\ref{fig:wildfb} shows the results on real human logs (WildFB). Despite real-world noise and fully off-policy data, UNO effectively extracts useful training signals and surpasses all other baselines.
This observation further highlights the necessity of UNO’s \textbf{cognitive gap assessment} and \textbf{clustering} mechanisms for robustly handling noisy user-generated data.

\textbf{Challenges of Learning from User Logs.} On certain datasets, such as Long-Long of MemoryBench, all other baselines fail to achieve positive gains and exhibit performance degradation. This is likely due to the logical consistency required by long-context tasks, where user biases and irrelevant noise in raw logs can easily divert the model’s attention. These results strongly support UNO’s central contribution: by identifying high-risk clusters via clustering and cognitive gap assessment, and addressing them through the \textbf{reflective path}, UNO safely incorporates valuable signals from user logs while preserving the model’s capabilities.

\begin{figure}
    \centering
    \includegraphics[width=\linewidth]{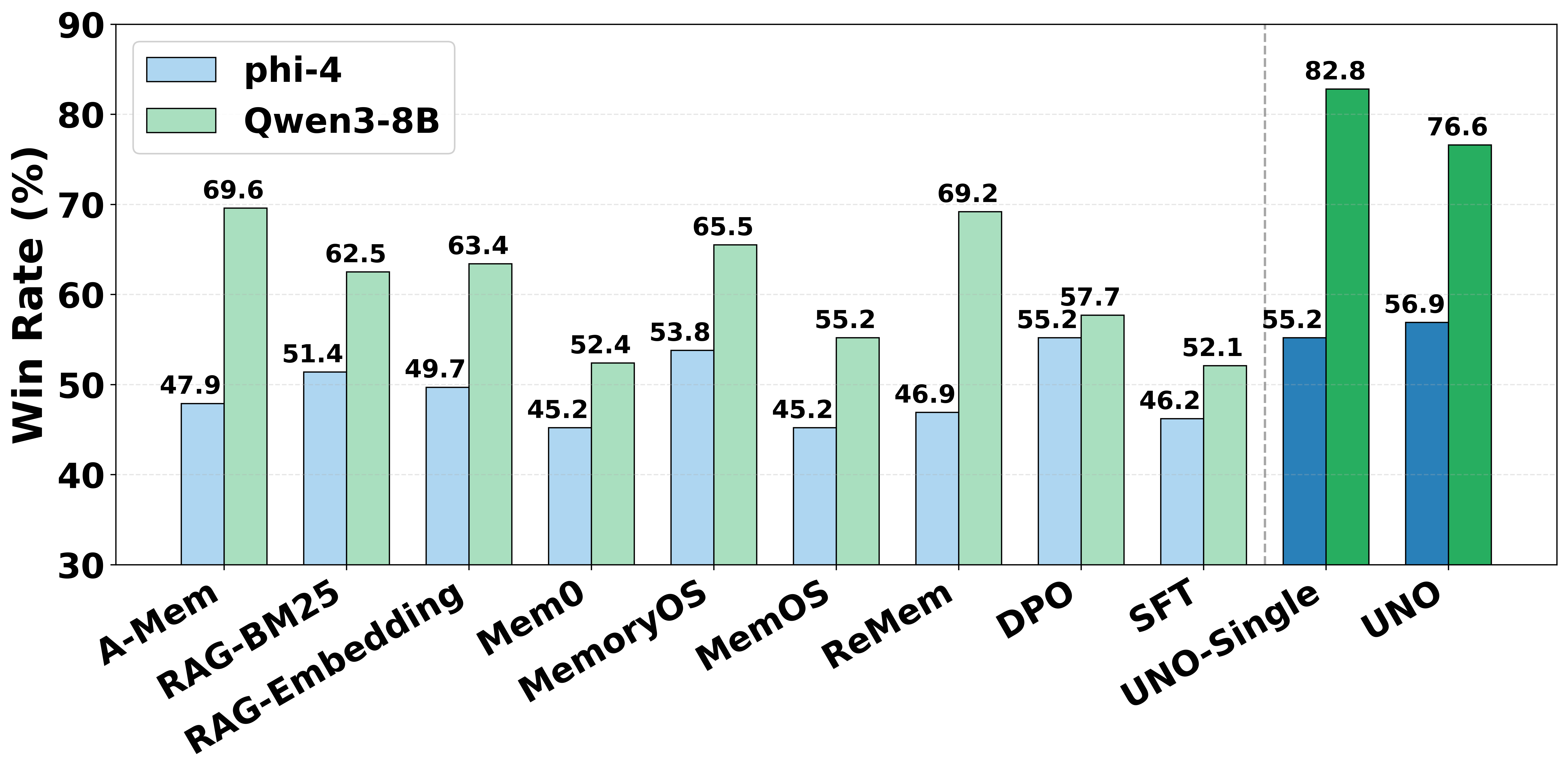}
    \caption{Main results in WildFB, measuring the win rate of generated against original responses.}
    \label{fig:wildfb}
\end{figure}

\begin{figure}[!t]
    \centering
    \includegraphics[width=0.98\linewidth]{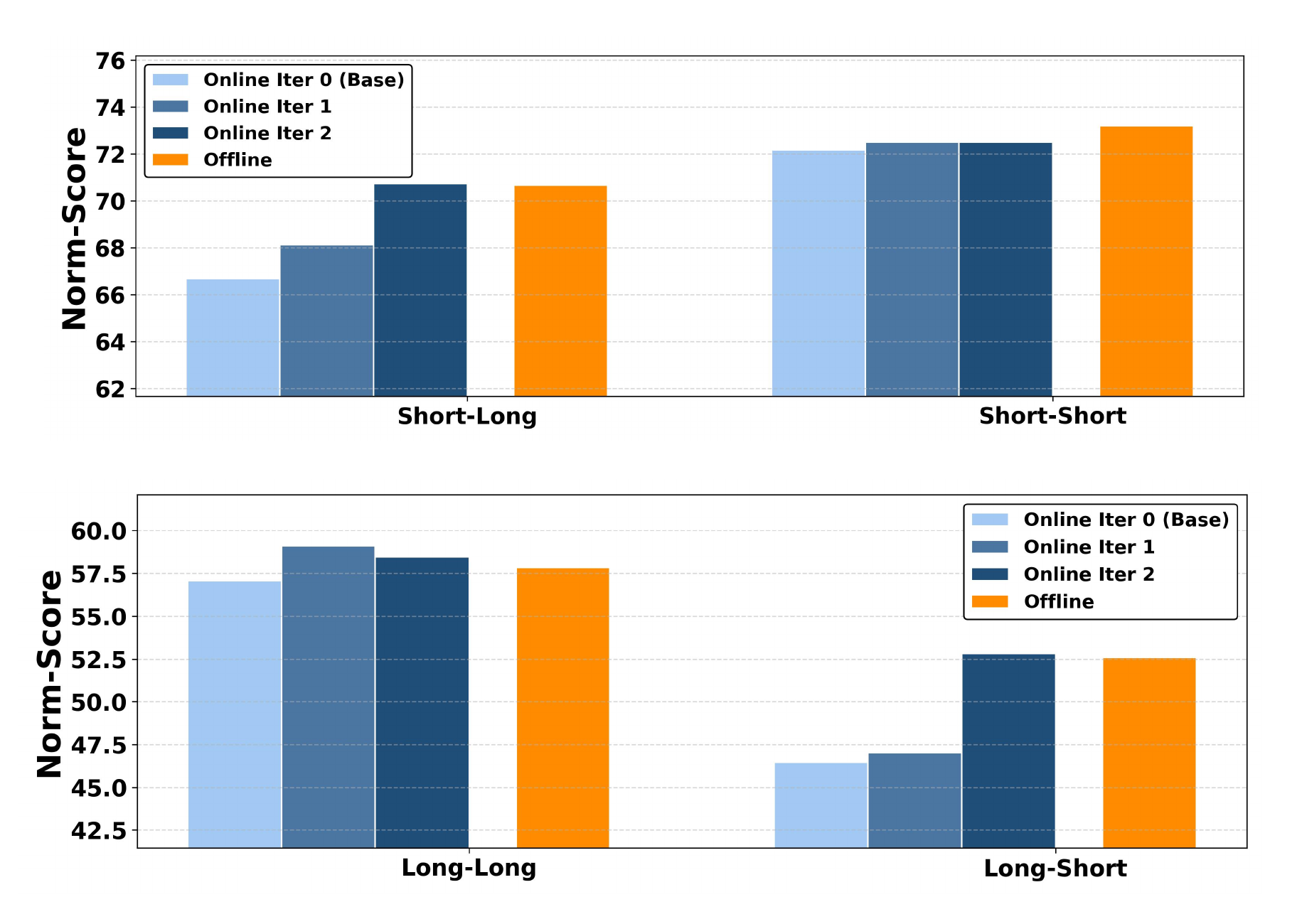}
    \caption{Results of online evolution settings on phi-4 model. "Offline" represents the setting in the main experiments.}
    \label{fig:online}
\end{figure}


\textbf{Trade-off between Efficiency and Performance.} Figure \ref{fig:token} illustrates the relationship between efficiency, measured by both extra input tokens and total inference time, and performance (Norm-Score). More details regarding the setup are provided in Appendix~\ref{apd:efficiency_details}. Because full UNO introduces additional output tokens and multiple LLM calls for clusters undergoing the reflective path, directly comparing token counts is complex. Thus, for the token analysis, we report UNO-Single (which retains only the primary path). UNO-Single occupies the optimal upper-left region of the plot, outperforming RAG and Memory approaches that rely heavily on retrieved context, while requiring zero additional input tokens. While traditional training shares this zero-overhead advantage, its performance remains highly sensitive to data noise. Furthermore, when evaluating total inference time, full UNO proves highly competitive. Although it may trigger multiple LLM calls, the actual average overhead is only 1.5 to 1.8 times the original, as these are selectively executed only for queries assigned to the reflective path. Consequently, full UNO still operates faster than most Memory baselines, while achieving the best performance among all methods. Together, UNO and UNO-Single provide an exceptional and flexible trade-off between efficiency and performance.

\textbf{Online Evolution Analysis.} We further evaluate UNO under an online evolution setting (described in Section~\ref{cha:online}) using phi-4, as shown in Figure~\ref{fig:online}. The results demonstrate that UNO naturally supports online iteration batch by batch: as user logs are incrementally incorporated, performance steadily improves on most tasks, underscoring the framework’s potential for lifelong learning. Notably, online and offline evolution exhibit complementary strengths. Offline evolution starts with a weaker initial model, often eliciting more negative feedback and thus yielding larger training sets. Online evolution, on the other hand, follows a ``learning while deploying'' paradigm. Although feedback volume may decrease as performance improves, the feedback is typically more informative.

\begin{table}[]
\centering
\caption{Ablation study of UNO's optimization framework (Qwen3-8B). UNO-Rfl and UNO-Prm refer to UNO with only the Reflective Experience Module and only the Primary Experience Module, respectively. The latter is identical to UNO-Single in the main experiments. ``Norm-S'' means Norm-Score.}
\resizebox{\columnwidth}{!}{%
\begin{tabular}{@{}lcccc@{}}
\toprule
\multirow{2}{*}{\textbf{}} & \multicolumn{2}{c}{\textbf{Short-Long}}                                      & \multicolumn{2}{c}{\textbf{Short-Short}}                \\ \cmidrule(l){2-5} 
                           & \multicolumn{1}{c|}{\textbf{Norm-S}} & \multicolumn{1}{c|}{\textbf{Z-Score}} & \multicolumn{1}{c|}{\textbf{Norm-S}} & \textbf{Z-Score} \\ \midrule
\textbf{RAG-Embedding (UNO's Rules)} & 71.50          & -29.88        & 67.74          & -22.14         \\
\textbf{RAG-BM25 (UNO's Rules)}      & 72.36          & -24.36        & 66.48          & -31.72         \\ \midrule
\textbf{UNO w/o Clustering}          & 70.95          & -32.82        & 74.73          & 9.84           \\ \midrule
\textbf{UNO-Rfl}                     & 75.44          & -3.38        & 72.79          & 3.34           \\
\textbf{UNO-Prm (UNO-Single)}        & 76.36          & 2.99          & 75.82          & 15.89          \\ \midrule
\textbf{UNO}                         & \textbf{77.09} & \textbf{7.16} & \textbf{76.26} & \textbf{21.54} \\ \bottomrule
\end{tabular}%
}
\label{tab:ablation_components}
\end{table}

\subsection{Ablation Study}

We perform comprehensive ablation studies of UNO mainly on both the Short-Long and Short-Short tasks for Qwen3-8B. Table~\ref{tab:ablation_components} summarizes the impact of removing or modifying components of the optimization framework, yielding the following insights:


\begin{itemize}[leftmargin=*]
\item \textbf{Clustering is critical}: Eliminating clustering (\textit{UNO w/o Clustering}) causes a drop in performance, which demonstrates that decoupling tasks through joint query- and rule-level features effectively simplifies the optimization process and provides a more reliable statistical basis for cognitive gap assessment.
\item \textbf{Complementarity of dual-path optimization}: We compare Primary Path only (\textit{UNO-Prm}, equivalent to \textit{UNO-Single} in Table~\ref{tab:main}) and Reflective Path only (\textit{UNO-Rfl}, where all clusters adopt Reflective Experience). Both variants are inferior to the full UNO with adaptive path selection. \textit{UNO-Prm} excels on low cognitive gap tasks but relies on the base LLM policy in noisier settings. While this avoids degradation, it is overly conservative and fails to exploit potentially valuable, high-difficulty log signals. Conversely, \textit{UNO-Rfl} increases inference latency (two generations and a Critic LoRA critique) and underperforms Primary Experience Construction on low cognitive gap tasks. A plausible reason is that the Critic LoRA learns to predict feedback from training distributions rather than reflecting genuine user feedback at test time. In low-gap cases, the noise it introduces outweighs its benefits. By selecting the appropriate path per cluster, the full UNO framework consistently outperforms single-path variants, confirming the necessity of dynamically switching strategies according to cognitive gaps and simulated verifying.
\item \textbf{Training is indispensable}: Although UNO distills user logs into compact rules, a naive alternative is to use these rules directly in a RAG pipeline, similar to the main experiments but replacing dialogue sessions with rules. As shown in Table~\ref{tab:ablation_components}, this approach performs even worse than the base model. Simply injecting rules as context can introduce noise, and rules from the training set, when retrieved via semantic matching, do not generalize to test queries. In contrast, internalizing knowledge through LoRA training or building specialized critic models provides a far more effective way to exploit training signals.
\end{itemize}

In Appendix~\ref{apd:abl}, we provide additional ablation studies, including the role of the cognitive gap assessment and ablation of the simulated verifier.

\section{Conclusion}

This work presents UNO (User log-driveN Optimization), a unified framework for the continual learning of LLMsys using raw user logs. By integrating dual-feature clustering and cognitive gap assessment, UNO effectively addresses the Signal-or-Noise Dilemma. The system adaptively selects between a Primary Experience Module for direct parameter optimization and a Reflective Experience Module for robust and critique-based refinement. Experimental results on MemoryBench confirm that UNO achieves state-of-the-art performance, which demonstrates how user logs of LLMs can be effectively harnessed as a rich resource for the continuous improvement of deployed LLM systems.

\section{Limitations}
\label{sec:limitations}

While UNO demonstrates state-of-the-art performance on existing benchmarks, our evaluations are primarily conducted on academic datasets. Although these datasets are carefully designed or derived from real chat systems, they do not fully encapsulate the extreme scale and diversity of real-world, industrial LLM services that may process millions of user sessions daily. We further discuss the feasibility of scaling UNO to such industrial scenarios in Appendix~\ref{sec:scalability}.  Furthermore, as noted in our related work, the development of user log-driven continual learning benchmarks is still in its infancy. We encourage the community to build more comprehensive, large-scale benchmarks to better evaluate LLMsys evolution in the wild.

Moreover, UNO's Reflective Path involves a multi-stage inference process. We have empirically shown that this design still outperforms most retrieval-based memory methods in terms of inference latency, offering a strong efficiency-performance trade-off (Figure~\ref{fig:token}). However, for strictly latency-constrained applications, this multi-stage generation inevitably introduces additional overhead. We outline a future direction to fold the Reflective Path into a single-stage generation process in Appendix~\ref{sec:scalability}.

\bibliography{custom}

\appendix

\section{Details of Methodology of UNO}

\subsection{Pseudo Code of Training Workflow of UNO}

We present the training pseudocode of UNO at Algorithm~\ref{alg}.

\begin{algorithm}[!t]
\caption{Training Framework of UNO}
\begin{algorithmic}[1]
\STATE \textbf{Input:} Initial policy $\pi_\theta$, User logs $\mathcal{D} = \{(q_i, y_i, \mathcal{F}_i)\}_{i=1}^N$, Thresholds $\tau^*, \gamma, \epsilon_{var}$
\STATE \textbf{Output:} Expert LoRAs $\{\Delta\theta_k\}$ and Critic LoRAs $\{\Delta\omega_k\}$

\STATE \textbf{/* Data Construction */}
\STATE $\mathcal{D}_{pref} \gets \emptyset$
\FOR{each $(q_i, y_i, \mathcal{F}_i) \in \mathcal{D}$}
    \STATE $\mathcal{R}_i \gets \text{Distill}(\mathcal{F}_i)$
    \IF{$\mathcal{R}_i \neq \emptyset$}
        \STATE Sample $y_i^w \sim \pi_\theta(\cdot \mid q_i, \mathcal{R}_i, y_i)$ 
        \STATE $\mathcal{D}_{pref} \gets \mathcal{D}_{pref} \cup \{(q_i, y_i^w, y_i, \mathcal{R}_i)\}$
    \ENDIF
\ENDFOR

\STATE \textbf{/* Dual-feature Clustering */}
\STATE $\{\mathcal{C}_k\}_{k=1}^K \gets \text{AgglomerativeClustering}(v_i, \epsilon_{var})$, where $v_i = [\text{Norm}(E(q_i)) \oplus \text{Norm}(E(\mathcal{R}_i))]$

\STATE \textbf{/* Adaptive Optimization for Different Clusters */}
\FOR{each cluster $\mathcal{C}_k$}
    \STATE \textbf{// Cognitive Gap Assessment}
    \FOR{each sample $i \in \mathcal{C}_k$}
        \STATE $\mathcal{R}_i^{LLM} \sim \pi_\theta(\cdot \mid q_i)$ 
        \STATE $g_i = \text{Dist}(\mathcal{R}_i^{LLM}, \mathcal{R}_i)$
    \ENDFOR
    \STATE $\mu_k = \frac{1}{|\mathcal{C}_k|} \sum_{i \in \mathcal{C}_k} g_i$
    
    \STATE \textit{status} $\gets$ \text{Reflective}
    \IF{$\mu_k \leq \tau^*$}
        \STATE \textbf{// Primary Experience Construction}
        \STATE $\Delta\theta_k \gets \arg\min_{\Delta} \mathbb{E}_{\mathcal{C}_k} [\mathcal{L}_{DPO}(\pi_{\theta+\Delta}; \pi_\theta) + \alpha \mathcal{L}_{NLL}(y^w)]$
        \STATE $BestWinRate \gets \text{SimulatedVerifier}(\pi_{\theta+\Delta\theta_k}, \pi_\theta, \mathcal{C}_{k}^{val})$ 
        \IF{$BestWinRate > \gamma$}
            \STATE \textit{status} $\gets$ \text{Primary}
        \ENDIF
    \ENDIF

    \IF{\textit{status} == \text{Reflective}}
        \STATE \textbf{// Reflective Experience Construction}
        \STATE \textbf{// Training a Critic LoRA for this cluster}
        \STATE $\Delta\omega_k \gets \arg\min_{\Delta} \mathbb{E}_{\mathcal{C}_k} [-\log \pi_{\theta+\Delta}(\mathcal{R}_i \mid q_i, y_i)]$
    \ENDIF
\ENDFOR
\end{algorithmic}
\label{alg}
\end{algorithm}

\subsection{Details of Dual-feature Data Clustering}
\label{apd:clustering}
In this subsection, we elaborate on the implementation details of the dual-feature clustering process introduced in Section~\ref{cha:clustering}. Our approach aims to construct a clustering space that captures both the user's semantic intent (query) and the domain-specific constraints (rules), while maintaining compatibility with inference scenarios where only the query is available.

\subsubsection{Feature Extraction and Fusion}
For a given data point containing a query $q_i$ and a corresponding rule set $\mathcal{R}_i$, we utilize a pre-trained sentence encoder (e.g., \texttt{Qwen3-Embedding-0.6B}) to extract dense representations. To ensure that both the query and the rule set contribute equally to the distance metrics during clustering, we apply L2-normalization independently to both embeddings before concatenation. The final dual-feature vector $v_i$ is formulated as:
$$ v_i = \text{Norm}\left( E(q_i) \oplus E(\mathcal{R}_i) \right) $$
where $E(\cdot)$ is the sentence encoder which produces normalized embeddings, $\text{Norm}(\cdot)$ denotes L2-normalization, and $\oplus$ is the concatenation operator. This final normalization step ensures that $v_i$ remains a unit vector, which is optimal for Euclidean distance-based clustering.

\subsubsection{Agglomerative Clustering and Ward's Linkage}
We employ Agglomerative Hierarchical Clustering to group the dual-feature vectors. A critical component of this algorithm is the linkage criterion, which determines the distance between sets of observations. We specifically adopt \textbf{Ward's Linkage} (Ward's minimum variance method)~\cite{müllner2011modernhierarchicalagglomerativeclustering}. 

Unlike other linkage methods (e.g., single or complete linkage) that measure distances between individual points in different clusters, Ward's linkage focuses on variance. At each step of the agglomerative process, the algorithm evaluates all possible pairs of clusters and merges the pair that results in the minimum increase in the total intra-cluster variance (or error sum of squares, ESS). Mathematically, the distance $d(u, v)$ between two clusters $u$ and $v$ in Ward's method is defined as the increase in ESS when they are merged:
$$ d(u, v) = \text{ESS}(u \cup v) - (\text{ESS}(u) + \text{ESS}(v)) $$
This linkage is particularly well-suited for our Euclidean space vectors, as it tends to produce compact and evenly sized clusters, effectively aligning samples with similar semantic intents and applicable rules. In our implementation, the agglomerative merging process stops once the scaled distance $d(u, v)$ exceeds a threshold of $4$.

\subsubsection{Asymmetric Centroid Calculation and Inference}
\label{apd:clusterinference}
A key innovation in our implementation is the decoupling of the clustering structure space and the inference space. While the cluster assignments $\mathcal{C}$ are derived using the dual-feature vectors $v_i$, the cluster centroids are computed \textit{exclusively} in the query embedding space. 

Specifically, for a cluster $c_k$, its centroid $\mu_k$ is calculated as the mean of the normalized query embeddings for all samples assigned to $c_k$:
$$ \mu_k = \frac{1}{|c_k|} \sum_{j \in c_k} \text{Norm}(E(q_j)) $$

This asymmetric design bridges the gap between training and inference. During the inference phase, a new user query $q_{new}$ typically arrives without rules, since the rule set is extracted from user feedback. We simply compute its query embedding $E(q_{new})$ and assign it to the cluster with the nearest centroid $\mu_k$ using Euclidean distance. Furthermore, we define a maximum distance threshold $d_{max}=1.2$, and if the distance to the nearest centroid exceeds this threshold, the query is flagged as an out-of-distribution class.

\section{Details of Theorems}
\subsection{Discussion on Assumption~\ref{assum:gap_dist}}
\label{apd:assumdiscuss}
The foundational premise of our assumption rests on the varying capacity of LLMs to differentiate noise from high-quality data across different cognitive boundaries. When the cognitive gap is small, high-quality instances generally represent incremental adjustments that align well with the model's existing knowledge representations, making them easily distinguishable from inherently unstructured noise. Conversely, as the cognitive gap widens, the instances increasingly diverge from the model's current capabilities. In these high-gap regions, we struggle to differentiate between truly informative, complex signals and meaningless random noise, resulting in the convergence of their cognitive distributions.

Empirical findings from deep learning research in natural language processing provide strong support for this intuition. In a study on dataset cartography~\cite{swayamdipta-etal-2020-dataset}, researchers map out dataset regions based on model behavior during training and identify a distinct category of instances that models find "hard-to-learn". Upon conducting qualitative analysis on these "hard-to-learn" instances, they observe that the underlying composition of this region is notably mixed. On one hand, this region captures a significant amount of mislabeled instances, which directly correspond to uninformative random noise. On the other hand, it simultaneously includes valid instances that are inherently challenging even for human annotators, representing complex, novel signals.
Crucially, because both genuine noise and valid, complex information are co-located in the exact same region from the model's perspective, the model is essentially blind to their differences. This empirical observation perfectly corroborates our assumption: in high-gap areas, the model's inability to reliably distinguish between informative novel signals and uninformative random noise causes the distributions of high-quality data and noisy data to inherently converge.

\subsection{Proof of Theorem~\ref{theo1}}
\label{apd:proof32}

\textbf{Theorem 3.2} (Noise Risk Bound)
\textit{For data with small cognitive gaps, the posterior probability of noise $P(N|g_i \leq \tau)$ is strictly bounded:}
\begin{equation}
P(N|g_i \leq \tau) \leq \frac{1}{1 + \alpha \frac{P_H(g_i \leq \tau)}{P_{N}(g_i \leq \tau)}}
\end{equation}

\begin{proof}
By Bayes’ theorem, the posterior probability that a sample is noise is given by
{\scriptsize
$$
P(N \mid g_i \leq \tau) = \frac{P(g_i \leq \tau \mid N)\, P(N)}{P(g_i \leq \tau \mid N)\, P(N) + P(g_i \leq \tau \mid H)\, P(H)} .
$$
}
Substituting $P(N)=1-\alpha$ and $P(H)=\alpha$ leads to
{\footnotesize
$$
P(N \mid g_i \leq \tau) = \frac{P_{\text{noise}}(g_i \leq \tau)(1-\alpha)}{P_{\text{noise}}(g_i \leq \tau)(1-\alpha) + P_H(g_i \leq \tau)\alpha} .
$$
}
The theorem holds trivially when $\alpha = 1$. When $\alpha < 1$, normalizing by dividing both the numerator and denominator by $P_{\text{noise}}(g_i \leq \tau)(1-\alpha)$ yields
{\footnotesize
$$
P(N \mid g_i \leq \tau) = \frac{1}{1 + \frac{\alpha}{1-\alpha} \frac{P_H(g_i \leq \tau)}{P_{\text{noise}}(g_i \leq \tau)}} .
$$
}
Since $1 - \alpha \in (0, 1]$, we have $\frac{\alpha}{1-\alpha}\geq \alpha$, which establishes Theorem~\ref{theo1}.
\end{proof}

\subsection{Proof of Theorem~\ref{theo2}}
\label{apd:proof33}

\begin{assumption}

The LLM rule-generation function $\Psi: \mathcal{Q} \to \mathcal{R}$ is $L_\Psi$-Lipschitz continuous. The semantic distance obeys the axioms of a metric (or at least the triangle inequality), and the function $Dist(\cdot, \cdot)$ used to evaluate the cognitive gap is treated as a semantic distance. Moreover, the embedding distance provides an upper bound on the semantic distance: $d_{\text{sem}}(A,B) \le C \cdot |E(A)-E(B)|$
\end{assumption}

\textbf{Theorem 3.3} (Variance Reduction via Clustering)
\textit{Under UNO’s dual-feature clustering, the intra-cluster variance of the cognitive gap $\text{Var}(g_i)$ is bounded by the cluster diameter $\epsilon$:}
\begin{equation}
\text{Var}(g_i) < [C \cdot (1 + L_\Psi) \cdot \epsilon]^2
\end{equation}
where $C$ and $L_\Psi$ are constants.

\begin{proof}
    Consider any two samples $i, j \in \mathcal{C}_k$ belonging to the same cluster. By the Lipschitz continuity of $\Psi$,
\[
d_{\text{sem}}(\mathcal{R}_i^{LLM}, \mathcal{R}_j^{LLM}) \le L_\Psi \cdot d_{\text{sem}}(q_i, q_j).
\]
For a cluster with diameter $\epsilon$,
\[
|(E(q_i) \oplus E(\mathcal R_i)) - (E(q_j) \oplus E(\mathcal R_j))| < \epsilon.
\]
This implies $|E(q_i)-E(q_j)| < \epsilon$ and $|E(\mathcal R_i)-E(\mathcal R_j)| < \epsilon$, and consequently
\[
d_{\text{sem}}(q_i, q_j) < C \cdot \epsilon,
\quad
d_{\text{sem}}(\mathcal R_i, \mathcal R_j) < C \cdot \epsilon.
\]
Thus,
{\footnotesize
\begin{equation}
\begin{aligned}
|g_i - g_j| &\leq |g_i - \text{Dist}(\mathcal{R}_i^{LLM}, \mathcal{R}_j)| \\
&\quad + |\text{Dist}(\mathcal{R}_i^{LLM}, \mathcal{R}_j) - g_j| \\
&= |\text{Dist}(\mathcal{R}_i, \mathcal{R}_i^{LLM}) - \text{Dist}(\mathcal{R}_j, \mathcal{R}_i^{LLM})| \\
&\quad + |\text{Dist}(\mathcal{R}_j, \mathcal{R}_i^{LLM}) - \text{Dist}(\mathcal{R}_j, \mathcal{R}_j^{LLM})| \\
&\leq \text{Dist}(\mathcal{R}_i, \mathcal{R}_j) + \text{Dist}(\mathcal{R}_i^{LLM}, \mathcal{R}_j^{LLM}) \\
& \le  C \cdot (1 + L_\Psi) \cdot \epsilon.
\end{aligned}
\end{equation}
}
By the definition of variance, Theorem~\ref{theo2} is proved.
\end{proof}

\section{Experimental Details}
\label{apd:detailsexp}

\subsection{Additional Evaluation Details}
\label{apd:evaluation_details}

\textbf{Introduction of MemoryBench.} MemoryBench spans multiple languages, tasks, domains, and datasets. Its log data are generated by a carefully designed and validated \textit{User Simulator} that produces interaction logs conditioned on each model’s own responses. The simulator outputs are validated via A/B testing, in which human annotators cannot distinguish simulator-generated logs from real user logs. The benchmark comprises two parts: user logs collected during the early stage (the training set) and subsequent new requests (the test set). All strategies are performed on the training set and evaluated on the test set. 

\textbf{Evaluation Metrics of MemoryBench.} Since test queries for each task originate from multiple datasets, MemoryBench first computes dataset-specific evaluation metrics and then applies either min–max normalization or z-score to obtain the final task-level performance. For both metrics, higher values correspond to better performance. We adopt the normalization parameters provided by the official repository.

\textbf{Details of Evaluation on WildFB.} We filter the dataset to include only initial user queries without prior dialogue history to ensure that user feedback remains on a single topic and to facilitate evaluation, retaining strictly English and Chinese conversations identified via language detection. We randomly sample 1,200 instances for the evolution logs (training set) and reserve the remainder for the evaluation tasks (test set), which aligns with the data scale of MemoryBench. For evaluation, we adopt WildReward-8B~\footnote{https://huggingface.co/THU-KEG/WildReward-8B} from the original WildFB paper as the evaluator and compute the win rate of the generated responses against the original responses in the dataset.

\subsection{Detailed Baseline Settings}

\begin{itemize}[leftmargin=*]
\item \textbf{Retrieval-Augmented Generation (RAG):} In our setting, we treat each dialogue session as an entry and use the test question as the search query to retrieve relevant user queries from the dialogues. We use either BM25~\cite{robertson2009probabilistic} or Qwen3-Embedding-0.6B~\cite{zhang2025qwen3embeddingadvancingtext} as the retriever and include entries relevant to the top-5 retrieved queries as the LLM context. When the context exceeds the LLM's maximum length, we truncate the documents using a bisection strategy.
\item \textbf{Memory for LLM Systems:} We evaluate MemOS~\cite{li2025memosmemoryosai}, ReMem~\cite{wei2025evomemorybenchmarkingllmagent}, A-Mem~\cite{xu2025amem}, Mem0~\cite{Mem0}, and MemoryOS~\cite{kang-etal-2025-memory}. All user logs are organized at the session level for memory storage, and the resulting systems are treated as evolved systems, uniformly using the top-5 entries. For the latter three methods, we follow all settings from the original MemoryBench paper. Although MemOS provides a complete API for the memory system, to ensure a consistent base LLM across all the methods, we use the official memory and retrieval APIs and evaluate on the test set with the same base LLM as other baselines. Note that the Long-Short dataset contains extremely long contexts, and we find that Mem0 cannot process them and cannot produce responses within a reasonable time. Thus, we do not report Mem0’s performance on this dataset. This finding is consistent with the conclusions reported in the original MemoryBench paper.
\item \textbf{Training Methods:} Using preference data constructed by UNO, we evaluate two offline training approaches: supervised fine-tuning (SFT) and DPO. All training settings follow UNO's settings. We select the epoch that achieves the lowest cross-entropy loss on the validation set as the final evaluation epoch.
\end{itemize}
\subsection{Details of UNO's Settings}
\label{app:uno_hyperparameters}

\textbf{Clustering and Training.} In agglomerative clustering, we stop merging clusters when the inter-cluster variance increment reaches 4. In the cluster matching process, distances greater than 1.2 are treated as outliers. We set the LoRA rank to 64 and the dropout rate to 0.05. Training is performed with the TRL library~\cite{vonwerra2022trl} using a learning rate of $5\times10^{-4}$ for 8 epochs. For DPO, we set the beta parameter to 0.1 and assign equal weights (0.5) to the DPO loss and the NLL loss.

\textbf{Evaluation and Thresholds.} For cognitive gap assessment, we take the minimum score across all rules as $\text{Dist}(\mathcal{R}_i, \mathcal{R}_i^{\text{LLM}})$. We simply set the threshold $\tau^*$ of the cognitive gap to 0.45, which is the mid-range of the experimental results for both models across all datasets on MemoryBench. In the LLM-as-Judge verifier, we sample the judge 3 times and average the scores (from 1 to 10). We set the win-rate threshold to 0.53 (0.03 above the base performance) and the BLEU threshold to 0.05; answers with BLEU scores below this threshold relative to the original responses are directly assigned a score of 0. We use vLLM~\cite{kwon2023efficient} for inference, and adopt the same inference settings as MemoryBench, with a temperature of 0.1.

\textbf{Online Evolution Criteria.} For unchanged clusters during online evolution, we perform continual training of the expert LoRA using incremental data, requiring the win rate to exceed the pre-evolution level by 0.03; otherwise, we retrain using the full dataset. For LoRA in the Reflective Experience Module, we retrain using the full data and require the best validation loss to decrease by more than 0.2 to ensure stable optimization. Due to the limited amount of distilled preference data for online training with Qwen3-8B in some datasets, we conduct the online experiments using phi-4.

\textbf{Prompt Templates.} The following prompt extracts raw user logs into a semi-structured set of rules.

\input{prompt/rule.tex}

The following prompt is used for revising the initital response according to the suggestions, both for data preprocessing and the reflective path at inference time.

\input{prompt/revise.tex}

The following prompt is used for judging the response according to the extracted rule set in the simulated validation module during training of the Primary Path.

\input{prompt/judge.tex}

\subsection{Details of Efficiency Evaluation}
\label{apd:efficiency_details}

For the calculation of efficiency metrics, all token counts are computed using the Qwen3-8B tokenizer. Regarding the measurement of inference time, our evaluation is conducted on GPUs with computational capabilities equivalent to the NVIDIA H100-SXM-80GB. The models are deployed using the vLLM framework distributed across 4 GPUs, with concurrency 10. We enable the batch-invariant setting during inference. 

Furthermore, our evaluation deliberately excludes the time consumed by offline processing phases, such as model training and the construction of retrieval indices. Instead, we focus exclusively on the online inference latency, which encompasses the time required for LLM generation, retrieval matching, cluster matching, and so on.

\section{Case Studies of UNO}

\input{case/inference.tex}

\section{Discussion on the Feasibility of Industrial-Grade Scalability}
\label{sec:scalability}

In our main experiments, UNO demonstrates state-of-the-art performance on existing academic-level benchmarks. However, a critical question arises when deploying UNO in real-world, industrial LLM services, where the system might process hundreds of millions of user interactions daily: \textit{How does the framework scale regarding cluster count, algorithmic complexity, and inference serving overhead?} To address this, we provide a theoretical discussion to explore the industrial feasibility of UNO, outlining proposed architectures and concepts for managing massive-scale deployments.

\paragraph{Quantitative Estimation of LoRA Modules}
Consider a hypothetical industrial scenario with $10^8$ daily user sessions. Even after rigorous pre-filtering (e.g., retaining only 1\% of sessions with explicit and actionable user feedback), the system would process approximately $10^6$ training samples. To maintain the fine-grained semantic and rule alignment that UNO requires, assuming an optimal capacity of $1,000$ to $5,000$ sessions per cluster, the system would generate approximately $200$ to $1,000$ distinct clusters. Consequently, this translates to maintaining and serving up to $1,000$ specialized LoRA modules (Expert or Critic).

\paragraph{Algorithmic Scalability of Clustering}
A primary bottleneck in scaling UNO directly is the clustering phase. The current implementation utilizes Agglomerative Clustering, which has a computational complexity ranging from $\mathcal{O}(N^2)$ to $\mathcal{O}(N^3)$. Applying this naively to $N=10^6$ samples is computationally intractable. To adapt UNO for such scales, we conceptualize and propose a two-stage hierarchical clustering pipeline as a highly scalable alternative. In the proposed first stage, a highly scalable algorithm such as an approximate nearest neighbor (ANN) search library like FAISS~\cite{douze2024faiss} could perform coarse-grained clustering to divide the massive log data into dozens of macro-clusters. In the theoretical second stage, the computationally intensive Agglomerative Clustering could be safely executed in parallel within each macro-cluster to perform fine-grained refinement, strictly bounding the maximum $N$ for any single agglomerative operation.

\paragraph{Scalable Multi-LoRA Serving}
To address the memory and routing overhead of matching incoming queries to up to $1,000$ potential LoRAs at inference time, naive sequential loading is inadequate. It is worth noting that the routing step itself is fundamentally a vector similarity matching process; since the total number of clusters is orders of magnitude smaller than the document corpora handled by standard text retrieval systems, identifying the correct cluster introduces virtually zero computational overhead. For LoRA loading, the subsequent memory and operational overhead of loading these weights can be effectively mitigated by leveraging the advanced multi-LoRA serving capabilities natively supported by modern inference frameworks like vLLM. Current vLLM implementations incorporate highly optimized multi-LoRA serving features by utilizing unified memory management and custom CUDA kernels to dynamically page active LoRA weights from CPU host memory to GPU VRAM. By fully exploiting the multi-LoRA infrastructure of modern inference frameworks, we hypothesize that UNO's multi-module architecture can efficiently route queries to the correct cluster and transition the corresponding experience module on the fly with negligible latency overhead, ensuring the feasibility and cost-effectiveness of UNO in production environments. In the small-scale experiment presented in Figure~\ref{fig:token}, we also observe that this overhead is negligible.

\paragraph{Folding the Reflective Path}
As noted in Section~\ref{sec:limitations}, the Reflective Path in UNO relies on a multi-stage inference process (initial generation, critique, and refinement). While we provide the single-stage \textit{UNO-Single} variant for strictly latency-constrained applications or industrial scenarios, an exciting avenue for future work is to fold the Reflective Path into a single-stage generation process via on-policy self-distillation~\cite{zhao2026selfdistilledreasoneronpolicyselfdistillation}. Specifically, the high-quality critique-and-refine trajectories generated offline could be used as distillation data to train a dedicated Expert LoRA. This approach may completely eliminate the multi-stage inference overhead at test time while preserving the robust performance and noise-resistance of the reflective experience.

\section{Further Ablation Study}
\label{apd:abl}
\begin{table}[]
\centering
\caption{Pre-filtering performance of the cognitive gap assessment for Qwen3-8B. The first row shows the proportion of reflective experience clusters, including those tagged during both the pre- and post-filtering stages. The second row quantifies the reduction in Primary Experience Module training attributable to pre-filtering. The final row reports the proportion of clusters identified early by pre-filtering; higher values indicate greater accuracy of pre-filtering in predicting failures of Primary Experience Module Construction.}
\resizebox{0.99\linewidth}{!}{%
\begin{tabular}{@{}lcccc@{}}
\toprule
                                    & \textbf{Short-Long} & \textbf{Short-Short} & \textbf{Long-Long} & \textbf{Long-Short} \\ \midrule
\textbf{\#Reflective Clus. / \#Clus.} & 2 / 3                 & 1 / 2                  & 1 / 4                & 2 / 3                 \\
\textbf{DPO Cost Savings}           & $\sim 53\%$         & $\sim 57\%$          & 0                  & $\sim 78\%$         \\
\textbf{Pre-Filter Recall}        & $50\%$              & $100\%$              & $0\%$              & $100\%$             \\ \bottomrule
\end{tabular}%
}
\label{tab:pre}
\end{table}

\begin{table}[]
\centering
\caption{Ablation study for simulated performance verifier. Perplexity and the DPO Chosen Reward are computed on the positive examples of the validation set.}
\resizebox{\columnwidth}{!}{%
\begin{tabular}{@{}lcccc@{}}
\toprule
\multirow{2}{*}{\textbf{Verification Criteria}}                       & \multicolumn{2}{c}{\textbf{Short-Long}} & \multicolumn{2}{c}{\textbf{Short-Short}} \\ \cmidrule(l){2-5} 
 & \multicolumn{1}{c|}{\textbf{Norm-S}} & \multicolumn{1}{c|}{\textbf{Z-S}} & \multicolumn{1}{c|}{\textbf{Norm-S}} & \textbf{Z-S} \\ \midrule
\textbf{Win-Rate in Simulated Verifier} & \textbf{77.09}      & \textbf{7.16}     & \textbf{76.26}      & \textbf{21.54}     \\
\textbf{Perplexity}                     & 71.13               & -29.49            & 72.97               & -6.49              \\
\textbf{DPO Chosen Reward}              & 76.30               & 2.07              & 72.97               & -6.49              \\
\textbf{Last Epoch}                     & 69.15               & -43.87            & 71.76               & -17.03             \\ \bottomrule
\end{tabular}
}
\label{tab:post}
\end{table}

\subsection{Cognitive Gap Assessment as Pre-Filtering}
Cognitive gap assessment and simulated performance verification serve as pre-filtering and post-filtering mechanisms for noisy user logs, respectively. Our analysis shows that even without pre-filtering, the verifier can also reliably filter out clusters with failed optimization, meaning cognitive gap assessment does not directly affect final model performance. Nevertheless, it significantly reduces the computational cost of Primary Experience Construction and simulated performance verification by accurately estimating log quality at the cluster level before training and guiding optimization decisions. Table~\ref{tab:pre} reports the number of clusters filtered at each stage, the total number of clusters produced, the reduction in DPO training volume due to cognitive gap assessment (measured by reduced training data), and the recall of pre-filtering (i.e., the proportion of clusters that should be filtered and are identified before training). These results demonstrate that cognitive gap assessment effectively evaluates log quality in advance and substantially reduces resource consumption.

\subsection{Cognitive Gap Assessment v.s. Model Confidence}

Estimating data difficulty via model confidence (e.g., Confident Learning~\cite{northcutt2022confidentlearningestimatinguncertainty}) is well-established in noisy label learning. However, directly applying token-level confidence (such as perplexity) to open-ended LLM generation often encounters calibration issues, as probabilities are confounded by generation formats and task openness rather than just data difficulty. Cognitive Gap addresses this by shifting the estimation from the probability space to the semantic reasoning space. Measuring the semantic distance between model-generated rules and those distilled from user logs is significantly more robust and computationally efficient, as it utilizes a smaller proxy model.

Figure~\ref{fig:cog} compares the classification performance for identifying Primary and Reflective Paths using Cognitive Gap versus negative log-likelihood (Qwen3-8B on MemoryBench). Since this task acts as a pre-filter, achieving 100\% recall for Primary Path clusters is crucial. At this decision boundary, model confidence misclassifies substantially more Reflective Path clusters into Primary Experience Module training, wasting computational resources. Cognitive Gap, however, yields significantly fewer errors. Furthermore, standard model confidence achieves an Area Under the Receiver Operating Characteristic curve (AUROC) of 0.8611, lower than the 0.8889 achieved by Cognitive Gap. This confirms that semantic-level metrics more reliably separate valuable feedback from noise in log-driven LLM optimization.

\begin{figure}[t]
    \centering
    \includegraphics[width=\linewidth]{fig/score_distribution.png}
    \caption{Comparison of the pre-filtering capabilities of Cognitive Gap and model confidence (mean negative log-probability). Scatter points represent individual clusters, and the dashed line indicates the decision boundary for completely filtering out clusters belonging to the Primary Path.}
    \label{fig:cog}
\end{figure}

\subsection{Simulated Verifier as Post-Filtering}

As a post-filter, the simulated performance verifier ensures quality control before deployment. We compare it with several naive validation strategies, including perplexity on positive validation samples, DPO chosen reward ($\beta \log \frac{\pi_{\theta}(y|x)}{\pi_{\text{ref}}(y|x)}$), and selecting the final training epoch. As shown in Table~\ref{tab:post}, verifier-based evaluation is more effective at identifying superior checkpoints or filtering failed clusters. Implicit metrics derived from positive samples perform poorly, likely because the verifier-based evaluation relies on generation-and-scoring, which is more robust and better aligned with generative performance.

\subsection{Ablation Studies on Key Hyperparameters}

Figure~\ref{fig:tau} presents a sensitivity analysis of the cognitive gap threshold, $\tau^*$. In the main experiments, we set $\tau^*$ to the mid-range value computed over all data from the two models, namely 0.45. With other hyperparameters fixed, we evaluate different threshold values on the MemoryBench dataset using Qwen3-8B. The results indicate that this hyperparameter is generally robust. Only on a few datasets (e.g., Long-Long) does an excessively small threshold slightly degrade performance. This robustness stems from the post-filtering mechanism: as long as the threshold is chosen conservatively (i.e., relatively large), performance remains unaffected, although the potential savings in training resources decrease. When the threshold is set too small, clusters that should be assigned to the primary path may instead be routed to the reflective path. Such misclassification has minimal impact on performance but substantially increases inference overhead.

Figure~\ref{fig:w} shows the impact of varying the win rate threshold in the Simulated Verifier. This hyperparameter likewise exhibits strong robustness. When the threshold is set exactly to 0.5, performance drops on certain datasets, such as Short-Long and Long-Long. This decline likely arises from the model’s inherent judgment capability: for samples with only marginal differences, decisions are more sensitive to noise and random perturbations, causing clusters formed at a 0.5 threshold to potentially reflect ineffective training. Conversely, setting the threshold too high may also misroute clusters that should belong to the primary path into the reflective path, thereby significantly increasing inference overhead.

\begin{figure}[t]
    \centering
    \includegraphics[width=\linewidth]{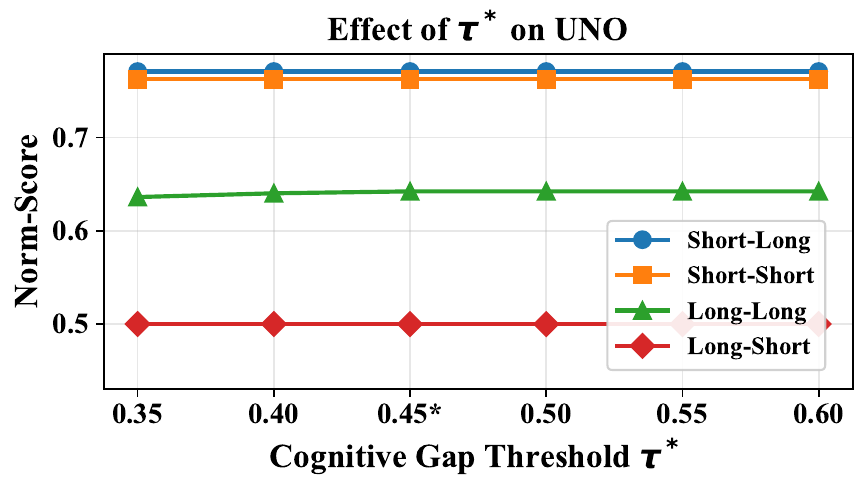}
    \caption{Ablation Study of Cognitive Gap Threshold $\tau^*$ on Qwen3-8B.}
    \label{fig:tau}
\end{figure}

\begin{figure}[t]
    \centering
    \includegraphics[width=\linewidth]{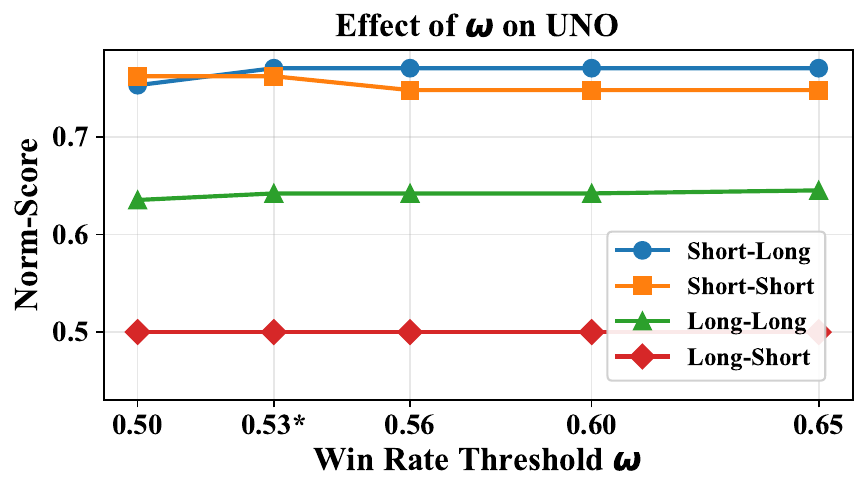}
    \caption{Ablation Study of Win Rate Threshold $w$ on Qwen3-8B.}
    \label{fig:w}
\end{figure}

\section{Detailed evaluation results on MemoryBench}

\label{apd:finegrained}

\input{tab/detail-q3}
\input{tab/detail-phi4.tex}

MemoryBench evaluates the overall performance of the LLMsys after optimization on the complete system log and therefore adopts aggregated relative metrics. Our main experiments follow this setting. In addition, we report fine-grained metrics for each sub-dataset, as presented in Table~\ref{tab:detailq3} and Table~\ref{tab:detailphi4}.

\end{document}

%% file: prompt/rule.tex
\begin{tcolorbox}[breakable,colback=lightgray!20,colframe=darkgray!80,title=System Prompt of Rule Set Generatioin from User Feedback.]
\#\# Role

You are a conversation analysis expert.
\\ \hspace*{\fill} \\ \relax
\#\# Task

Your task is to analyze a conversation snippet that includes an initial user query, the model's initial response, and subsequent user feedback. Based on your analysis, you need to summarize and output the following two key pieces of information in JSON format:
\\ \hspace*{\fill} \\ \relax
1. Whether the user is satisfied with the model's initial response.

2. Extract rules or suggestions from the user's feedback that can help the model better answer the initial query.
\\ \hspace*{\fill} \\ \relax
\#\# Input Format

You will receive a text containing the following three parts:
\\ \hspace*{\fill} \\ \relax
[Initial User Query]

[Initial Model Response]

[User Feedback] (May contain feedback from one or more rounds of conversation)
\\ \hspace*{\fill} \\ \relax
\#\# Output Requirements

You must strictly follow the JSON format below for your output:
\\ \hspace*{\fill} \\ \relax
\{

  "rules": [

    "<string>",

    "<string>",

    ...

  ]

\}
\\ \hspace*{\fill} \\ \relax
\#\#\# Field Descriptions

- rules (list of strings): A list containing all the effective suggestions or rules extracted from the user feedback.
\\ \hspace*{\fill} \\ \relax
\#\# Important Constraints
\\ \hspace*{\fill} \\ \relax
1. Only extract feedback that is directly related to the [Initial User Query]. If the user's replies deviate from the initial question (e.g., start small talk or ask a new question), ignore this irrelevant content.

2. If all feedback is unrelated to the initial question, or if the user expresses satisfaction without providing any specific suggestions, this list can be empty [].

3. Each rule should be a clear, actionable instruction. It should be in affirmative form, stating what to do rather than what not to do.
\end{tcolorbox}

%% file: prompt/revise.tex
\begin{tcolorbox}[breakable,colback=lightgray!20,colframe=darkgray!80,title=Prompt of Response Revision]
\#\# Initial Answer

\{old\_answer\}
\\ \hspace*{\fill} \\ \relax
\#\# Suggestion

\{suggestion\}
\\ \hspace*{\fill} \\ \relax
\#\# User Question

\{question\}
\\ \hspace*{\fill} \\ \relax
Follow the suggestion to revise the initial answer. Directly output the revised complete answer that accurately and naturally addresses the user's question. Don't repeat the question itself at the beginning of your response.
\end{tcolorbox}

%% file: prompt/judge.tex
\begin{tcolorbox}[breakable,colback=lightgray!20,colframe=darkgray!80,title=Prompt of LLM-AS-JUDGE of Simulated Validation]
\#\#\# ROLE

You are a rigorous and objective AI Answer Quality Evaluator.
\\ \hspace*{\fill} \\ \relax
\#\#\# TASK

Your task is to comprehensively evaluate the quality of the provided [Answer] based on the [Question] and the [Expert Suggestion]. You must output a final integer score from 1 to 10.
\\ \hspace*{\fill} \\ \relax
\#\#\# CORE INSTRUCTIONS
\\ \hspace*{\fill} \\ \relax
1.  **[Expert Suggestion] is the Primary Standard**: You must strictly adhere to the [Expert Suggestion]. It outlines the core direction, key points to include, or critical errors to avoid for a high-quality answer. The degree to which the [Answer] follows the [Expert Suggestion] is the most critical factor in determining its score.
\\ \hspace*{\fill} \\ \relax
2.  **Comprehensive Judgment**: While the [Expert Suggestion] is paramount, it may not cover all evaluation dimensions. Therefore, you must also apply your own knowledge and judgment to assess the answer based on these supplementary criteria:

    * **Accuracy**: Is the information in the answer factually correct and free of errors?

    * **Completeness**: Does the answer fully address all parts of the [Question]?

    * **Clarity \& Logic**: Is the answer well-structured, easy to understand, and logically sound?

    * **Relevance**: Does the answer stay focused on the [Question] without including irrelevant information?
\\ \hspace*{\fill} \\ \relax
\#\#\# SCORING RUBRIC
\\ \hspace*{\fill} \\ \relax
Use the following guidelines to assign an integer score from 1 to 10:
\\ \hspace*{\fill} \\ \relax
* **9-10 (Excellent):**

    * Exemplary adherence to all key points in the [Expert Suggestion].

    * The answer is highly accurate, comprehensive, and clearly articulated.

    * May provide additional valuable insights beyond the direct scope of the question.
\\ \hspace*{\fill} \\ \relax
* **7-8 (Good):**

    * Strongly adheres to the core points of the [Expert Suggestion].

    * The answer is factually correct and effectively addresses the question.

    * May have minor imperfections, such as a slight lack of detail or minor stylistic issues.
\\ \hspace*{\fill} \\ \relax
* **5-6 (Average):**

    * Partially adheres to the [Expert Suggestion] but with noticeable deviations or omissions.

    * Answers the main part of the [Question] but contains some inaccuracies, is incomplete, or lacks clarity.

    * The answer is fundamentally acceptable but has clear room for improvement.
\\ \hspace*{\fill} \\ \relax
* **3-4 (Poor):**

    * Largely disregards or contradicts the [Expert Suggestion].

    * Contains significant factual errors, logical flaws, or is substantially incomplete.

    * Fails to effectively answer the [Question].
\\ \hspace*{\fill} \\ \relax
* **1-2 (Very Poor):**

    * Completely fails to follow the [Expert Suggestion].

    * The answer is factually incorrect, irrelevant to the question, nonsensical, or potentially harmful.
\\ \hspace*{\fill} \\ \relax
\#\#\# OUTPUT FORMAT
\\ \hspace*{\fill} \\ \relax
Your output **MUST** be a single JSON object that strictly follows the format below. Do not add any extra text, comments, or explanations before or after the JSON code block.
\\ \hspace*{\fill} \\ \relax
```json

\{

  "reason": "str",

  "score": "int"

\}

```
\\ \hspace*{\fill} \\ \relax
- reason (str): Provide a detailed and specific explanation for the score you have given. In your reasoning, you must explicitly reference the [Expert Suggestion] and describe the extent to which the [Answer] adhered to it. Also, incorporate your assessment of other criteria (e.g., accuracy, completeness).

- score (int): Your integer score between 1 and 10.
\\ \hspace*{\fill} \\ \relax
\#\#\# EVALUATION CONTENT

\#\#\#\# Question

\{question\}
\\ \hspace*{\fill} \\ \relax
\#\#\#\# Expert Suggestion

\{suggestion\}
\\ \hspace*{\fill} \\ \relax
\#\#\#\# Answer

\{answer\}
\end{tcolorbox}

%% file: case/inference.tex
Figure~\ref{fig:case1} illustrates how UNO transforms raw user logs into training data. The ``Extracted Rule Set'' comprises semi-structured revision suggestions derived from user feedback. Conditioned on the initial response, the user feedback, and the rule set, the model generates a revised response. The initial and revised responses are treated as the rejected and chosen responses, respectively, for DPO training.

Figure~\ref{fig:case2} presents the inference workflow of UNO. The Reflective Path loads the Critic LoRA, produces suggestions based on the initial response, and then prompts the base LLM to generate a revision. In contrast, the Primary Path directly loads the Expert LoRA to produce the response. In this case, both approaches outperform the base LLM’s initial output. The Reflective Path incorporates more detailed considerations in the legal judgment, whereas the Primary Path appends the complete legal provisions, thereby rendering the judgment document more comprehensive.

\begin{figure*}
    \centering
    \includegraphics[width=\textwidth]{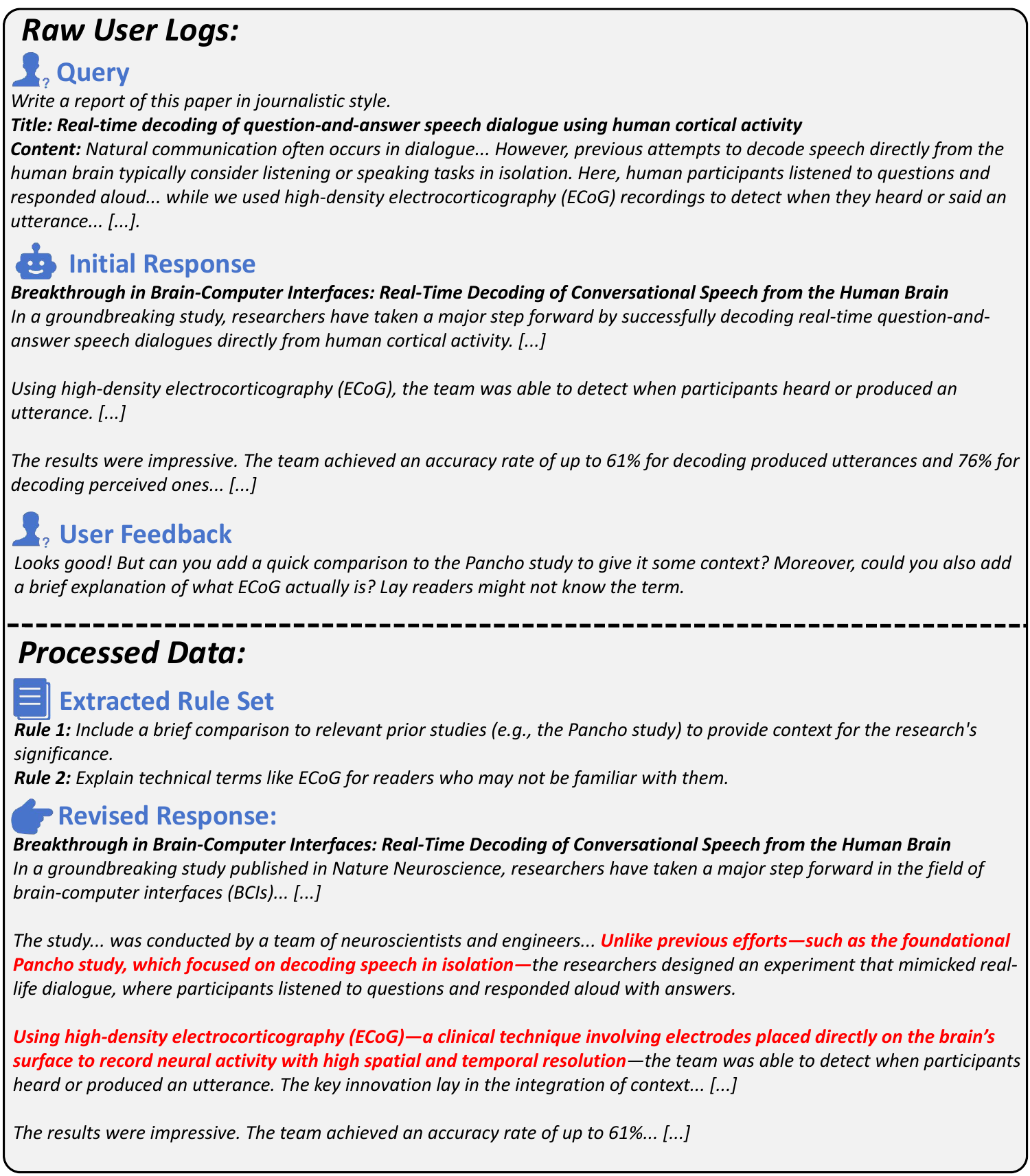}
    \caption{Case of Training Data Preprocessing. This item is from JuDGE dataset, which belongs to the Short-Input-Long-Output subset of MemoryBench.}
    \label{fig:case1}
\end{figure*}

\begin{figure*}
    \centering
    \includegraphics[width=\textwidth]{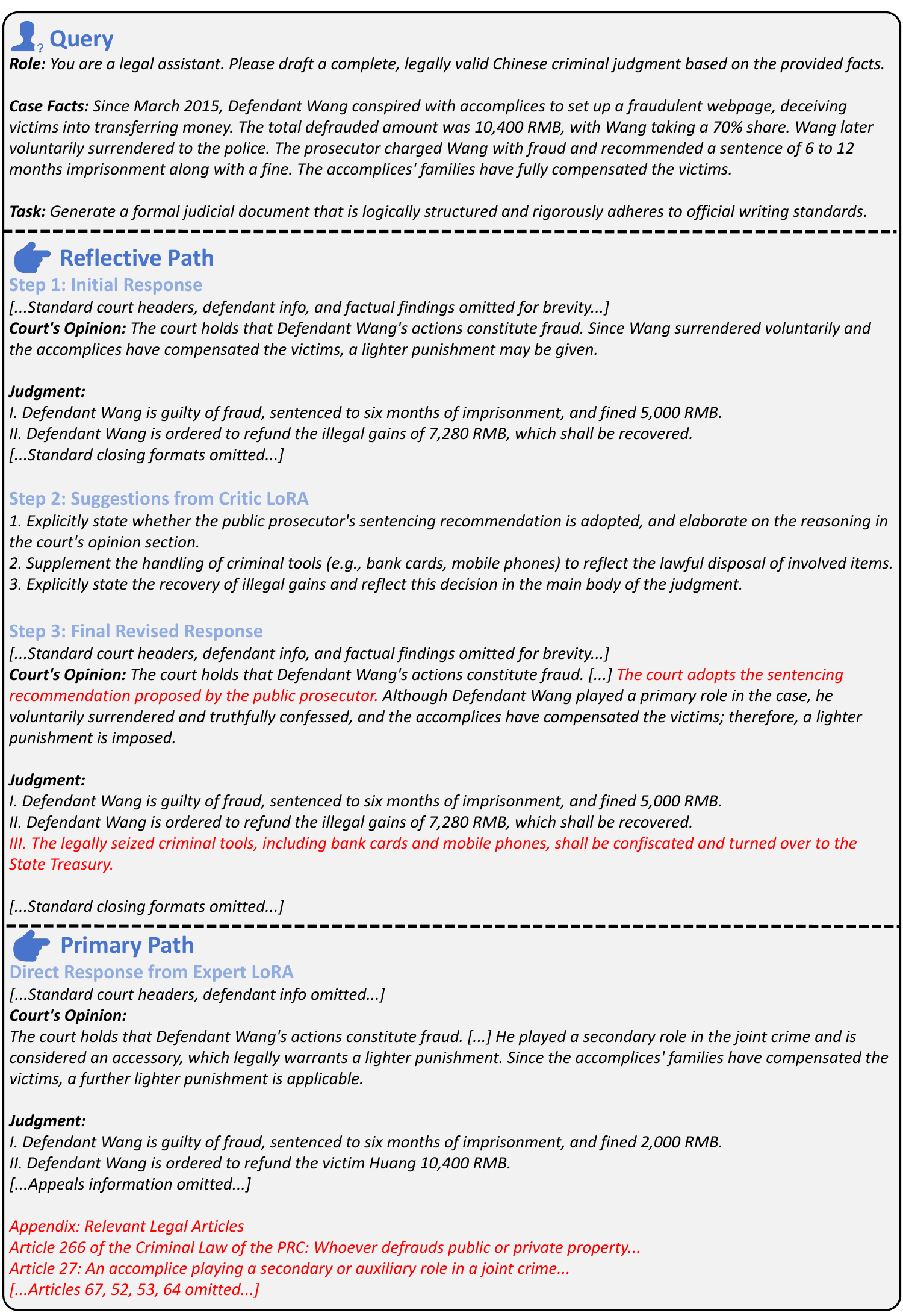}
    \caption{Case of Inference Workflow. This item is from SciTechNews dataset, which belongs to the Short-Input-Short-Output subset of MemoryBench.}
    \label{fig:case2}
\end{figure*}

\label{apd:casestudy}

%% file: tab/detail-q3.tex
\begin{table*}[]
\centering
\caption{Detailed evaluation results of LLMsys with Qwen3-8B backbone on MemoryBench. For each dataset, we report raw performances using their dataset-specific evaluation metrics. Except for metrics marked with↓, higher values indicate better performance.}
\resizebox{\textwidth}{!}{%
\begin{tabular}{@{}|c|c|c|c|c|c|c|c|c|c|c|c|c|c|c|@{}}
\toprule
\textbf{Type} &
  \textbf{Dataset} &
  \textbf{Metric} &
  \textbf{Base} &
  \textbf{Embedding} &
  \textbf{BM25} &
  \textbf{MemOS} &
  \textbf{ReMem} &
  \textbf{A-Mem} &
  \textbf{Mem0} &
  \textbf{MemoryOS} &
  \textbf{DPO} &
  \textbf{SFT} &
  \textbf{UNO-Single} &
  \textbf{UNO} \\ \midrule
\multirow{7}{*}{\rotatebox{90}{\textbf{Short-Short}}} &
  \multirow{5}{*}{\textbf{SciTechNews}} &
  \textbf{BERTScore-F1} &
  0.8179 &
  0.8154 &
  0.8151 &
  0.8192 &
  0.8165 &
  0.8171 &
  0.8192 &
  0.8127 &
  0.8242 &
  0.8256 &
  0.8179 &
  0.8170 \\ \cmidrule(l){3-15} 
 &
   &
  \textbf{CLI↓} &
  16.4672 &
  16.0973 &
  15.9736 &
  16.3992 &
  16.3043 &
  16.9455 &
  16.4013 &
  15.6473 &
  16.5301 &
  17.2997 &
  16.4672 &
  16.1465 \\ \cmidrule(l){3-15} 
 &
   &
  \textbf{DCRS↓} &
  12.6027 &
  12.5169 &
  12.4130 &
  12.6677 &
  12.5773 &
  12.7851 &
  12.6697 &
  12.3047 &
  12.7092 &
  13.1793 &
  12.6027 &
  12.4879 \\ \cmidrule(l){3-15} 
 &
   &
  \textbf{FKGL↓} &
  16.0443 &
  15.8845 &
  15.8840 &
  16.2397 &
  16.0830 &
  16.4449 &
  16.0793 &
  15.2932 &
  16.2405 &
  17.1901 &
  16.0443 &
  15.5792 \\ \cmidrule(l){3-15} 
 &
   &
  \textbf{Rouge-L} &
  0.1212 &
  0.1230 &
  0.1215 &
  0.1225 &
  0.1163 &
  0.1205 &
  0.1242 &
  0.1072 &
  0.1361 &
  0.1418 &
  0.1212 &
  0.1192 \\ \cmidrule(l){2-15} 
 &
  \textbf{LexEval-QA} &
  \textbf{Rouge-L} &
  0.1111 &
  0.1128 &
  0.1155 &
  0.0840 &
  0.1112 &
  0.1312 &
  0.1050 &
  0.1206 &
  0.1350 &
  0.1649 &
  0.1451 &
  0.1451 \\ \cmidrule(l){2-15} 
 &
  \textbf{NFCats} &
  \textbf{Score} &
  4.5600 &
  4.5600 &
  4.5000 &
  4.1400 &
  4.3600 &
  4.1000 &
  3.2600 &
  4.2000 &
  4.5200 &
  4.4200 &
  4.6000 &
  4.6000 \\ \midrule
\multirow{12}{*}{\rotatebox{90}{\textbf{Short-Long}}} &
  \textbf{\begin{tabular}[c]{@{}c@{}}HelloBench\\ -Academic\\ \&Knowledge\\ -QA\end{tabular}} &
  \textbf{Avg. Score} &
  0.8332 &
  0.8483 &
  0.8488 &
  0.8290 &
  0.8272 &
  0.8458 &
  0.8610 &
  0.8375 &
  0.8625 &
  0.8693 &
  0.8768 &
  0.8768 \\ \cmidrule(l){2-15} 
 &
  \multirow{10}{*}{\textbf{JuDGE}} &
  \textbf{Amount Score} &
  0.5650 &
  0.5481 &
  0.5993 &
  0.5870 &
  0.5059 &
  0.5520 &
  0.2975 &
  0.5554 &
  0.4302 &
  0.4992 &
  0.5650 &
  0.5585 \\ \cmidrule(l){3-15} 
 &
   &
  \textbf{Crime Precision} &
  0.9700 &
  0.9700 &
  0.9500 &
  0.9600 &
  0.9600 &
  0.9100 &
  0.9100 &
  0.9600 &
  0.9200 &
  0.9400 &
  0.9700 &
  0.9600 \\ \cmidrule(l){3-15} 
 &
   &
  \textbf{Crime Recall} &
  1.0000 &
  1.0000 &
  0.9600 &
  0.9800 &
  0.9800 &
  0.9200 &
  0.9200 &
  0.9800 &
  0.9400 &
  0.9600 &
  1.0000 &
  0.9800 \\ \cmidrule(l){3-15} 
 &
   &
  \textbf{\begin{tabular}[c]{@{}c@{}}Judge\\ Bert Score\end{tabular}} &
  0.7901 &
  0.7720 &
  0.7683 &
  0.7806 &
  0.7588 &
  0.7646 &
  0.4544 &
  0.7486 &
  0.8671 &
  0.8079 &
  0.7901 &
  0.7907 \\ \cmidrule(l){3-15} 
 &
   &
  \textbf{Judge Meteor} &
  0.4009 &
  0.3807 &
  0.3755 &
  0.4003 &
  0.3663 &
  0.4084 &
  0.2900 &
  0.3257 &
  0.5987 &
  0.4617 &
  0.4009 &
  0.4052 \\ \cmidrule(l){3-15} 
 &
   &
  \textbf{\begin{tabular}[c]{@{}c@{}}Penalcode\\ Index Precision\end{tabular}} &
  0.7294 &
  0.7200 &
  0.7065 &
  0.7189 &
  0.7010 &
  0.7592 &
  0.6011 &
  0.6721 &
  0.7970 &
  0.6572 &
  0.7294 &
  0.7413 \\ \cmidrule(l){3-15} 
 &
   &
  \textbf{\begin{tabular}[c]{@{}c@{}}Penalcode\\ Index Recall\end{tabular}} &
  0.7476 &
  0.7385 &
  0.7989 &
  0.7399 &
  0.7316 &
  0.7748 &
  0.3382 &
  0.7390 &
  0.6215 &
  0.6581 &
  0.7476 &
  0.7516 \\ \cmidrule(l){3-15} 
 &
   &
  \textbf{\begin{tabular}[c]{@{}c@{}}Reasoning\\ Bert Score\end{tabular}} &
  0.8140 &
  0.8147 &
  0.8160 &
  0.7723 &
  0.8147 &
  0.7844 &
  0.7049 &
  0.8138 &
  0.8154 &
  0.8069 &
  0.8140 &
  0.8081 \\ \cmidrule(l){3-15} 
 &
   &
  \textbf{\begin{tabular}[c]{@{}c@{}}Reasoning\\ Meteor\end{tabular}} &
  0.5101 &
  0.4861 &
  0.5133 &
  0.4580 &
  0.5007 &
  0.4974 &
  0.3621 &
  0.4895 &
  0.4940 &
  0.4946 &
  0.5101 &
  0.5186 \\ \cmidrule(l){3-15} 
 &
   &
  \textbf{Time Score} &
  0.7224 &
  0.7237 &
  0.7108 &
  0.7057 &
  0.6834 &
  0.6910 &
  0.5071 &
  0.7139 &
  0.7049 &
  0.6627 &
  0.7224 &
  0.7168 \\ \cmidrule(l){2-15} 
 &
  \textbf{WritingPrompts} &
  \textbf{Meteor} &
  0.2367 &
  0.2245 &
  0.2342 &
  0.2238 &
  0.2251 &
  0.2056 &
  0.2209 &
  0.2425 &
  0.2008 &
  0.1998 &
  0.2367 &
  0.2451 \\ \midrule
\multirow{6}{*}{\rotatebox{90}{\textbf{Long-Long}}} &
  \textbf{\begin{tabular}[c]{@{}c@{}}HelloBench\\ -Academic\\ \&Knowledge\\ -Writing\end{tabular}} &
  \textbf{Avg. Score} &
  0.8515 &
  0.8319 &
  0.8250 &
  0.8333 &
  0.8353 &
  0.8240 &
  0.8206 &
  0.7951 &
  0.8377 &
  0.8319 &
  0.8211 &
  0.8294 \\ \cmidrule(l){2-15} 
 &
  \textbf{\begin{tabular}[c]{@{}c@{}}HelloBench\\ -Creative\&Design\end{tabular}} &
  \textbf{Avg. Score} &
  0.8291 &
  0.7624 &
  0.7888 &
  0.7142 &
  0.7633 &
  0.7895 &
  0.7791 &
  0.7578 &
  0.7711 &
  0.7894 &
  0.8246 &
  0.8250 \\ \cmidrule(l){2-15} 
 &
  \textbf{LexEval-Judge} &
  \textbf{Rouge-L} &
  0.0753 &
  0.0895 &
  0.0628 &
  0.0238 &
  0.0695 &
  0.0676 &
  0.0235 &
  0.0082 &
  0.0700 &
  0.0843 &
  0.0901 &
  0.0901 \\ \cmidrule(l){2-15} 
 &
  \textbf{\begin{tabular}[c]{@{}c@{}}WritingBench\\ -Academic\\ \&Engineering\end{tabular}} &
  \textbf{Score} &
  6.9706 &
  6.7353 &
  7.0882 &
  6.5294 &
  6.0588 &
  6.8235 &
  6.6471 &
  6.0294 &
  6.7647 &
  6.7647 &
  7.1176 &
  7.1176 \\ \cmidrule(l){2-15} 
 &
  \textbf{\begin{tabular}[c]{@{}c@{}}WritingBench\\ -Creative\&Design\end{tabular}} &
  \textbf{Score} &
  6.6744 &
  6.6512 &
  6.5698 &
  6.6977 &
  6.4535 &
  6.4651 &
  6.4884 &
  4.8953 &
  6.6744 &
  6.6512 &
  6.6512 &
  6.7326 \\ \cmidrule(l){2-15} 
 &
  \textbf{\begin{tabular}[c]{@{}c@{}}WritingBench\\ -Politics\&Law\end{tabular}} &
  \textbf{Score} &
  6.8780 &
  6.8537 &
  6.8537 &
  6.6341 &
  6.6585 &
  6.6829 &
  6.6829 &
  4.6829 &
  6.8780 &
  6.8049 &
  6.9756 &
  6.9756 \\ \midrule
\multirow{7}{*}{\rotatebox{90}{\textbf{Long-Short}}} &
  \multirow{4}{*}{\textbf{IdeaBench}} &
  \textbf{Bert Score} &
  0.5605 &
  0.5676 &
  0.5675 &
  0.5418 &
  0.5632 &
  0.5576 &
  - &
  0.5553 &
  0.5681 &
  0.5672 &
  0.5605 &
  0.5655 \\ \cmidrule(l){3-15} 
 &
   &
  \textbf{\begin{tabular}[c]{@{}c@{}}LLM Feasibility\\ Ranking Score\end{tabular}} &
  0.1267 &
  0.1867 &
  0.1733 &
  0.0133 &
  0.1267 &
  0.1400 &
  - &
  0.0867 &
  0.1000 &
  0.0867 &
  0.1267 &
  0.1533 \\ \cmidrule(l){3-15} 
 &
   &
  \textbf{\begin{tabular}[c]{@{}c@{}}LLM Novelty\\ Ranking Score\end{tabular}} &
  0.7600 &
  0.7667 &
  0.7400 &
  0.5067 &
  0.7133 &
  0.6400 &
  - &
  0.4533 &
  0.7067 &
  0.6600 &
  0.7600 &
  0.7400 \\ \cmidrule(l){3-15} 
 &
   &
  \textbf{\begin{tabular}[c]{@{}c@{}}LLM \\ Rating Score\end{tabular}} &
  4.8800 &
  4.8800 &
  4.7000 &
  3.0800 &
  4.8800 &
  4.6000 &
  - &
  4.8800 &
  4.2400 &
  4.7000 &
  4.8800 &
  4.7600 \\ \cmidrule(l){2-15} 
 &
  \textbf{\begin{tabular}[c]{@{}c@{}}LexEval\\ -Summarization\end{tabular}} &
  \textbf{Rouge-L} &
  0.2237 &
  0.2306 &
  0.2256 &
  0.2191 &
  0.2205 &
  0.2351 &
  - &
  0.2400 &
  0.2191 &
  0.2295 &
  0.2218 &
  0.2218 \\ \cmidrule(l){2-15} 
 &
  \multirow{2}{*}{\textbf{LimitGen-Syn}} &
  \textbf{Accuracy} &
  0.4400 &
  0.4600 &
  0.5000 &
  0.6000 &
  0.3800 &
  0.5000 &
  - &
  0.1400 &
  0.5000 &
  0.5600 &
  0.4400 &
  0.5600 \\ \cmidrule(l){3-15} 
 &
   &
  \textbf{Rating} &
  1.1400 &
  1.2200 &
  1.2800 &
  1.8000 &
  0.9400 &
  1.3200 &
  - &
  0.3400 &
  1.3200 &
  1.4800 &
  1.1400 &
  1.5400 \\ \bottomrule
\end{tabular}%
}
\label{tab:detailq3}
\end{table*}

%% file: tab/detail-phi4.tex
\begin{table*}[]
\centering
\caption{Detailed evaluation results of LLMsys with phi-4 backbone on MemoryBench. For each dataset, we report raw performances using their dataset-specific evaluation metrics. Except for metrics marked with↓, higher values indicate better performance.}
\resizebox{\textwidth}{!}{%
\begin{tabular}{@{}|c|c|c|c|c|c|c|c|c|c|c|c|c|c|c|@{}}
\toprule
\textbf{Type} &
  \textbf{Dataset} &
  \textbf{Metric} &
  \textbf{Base} &
  \textbf{Embedding} &
  \textbf{BM25} &
  \textbf{MemOS} &
  \textbf{ReMem} &
  \textbf{A-Mem} &
  \textbf{Mem0} &
  \textbf{MemoryOS} &
  \textbf{DPO} &
  \textbf{SFT} &
  \textbf{UNO-Single} &
  \textbf{UNO} \\ \midrule
\multirow{7}{*}{\rotatebox{90}{\textbf{Short-Short}}} &
  \multirow{5}{*}{\textbf{SciTechNews}} &
  \textbf{BERTScore-F1} &
  0.8280 &
  0.8261 &
  0.8268 &
  0.8239 &
  0.8257 &
  0.8244 &
  0.8248 &
  0.8081 &
  0.8271 &
  0.8304 &
  0.8280 &
  0.8220 \\ \cmidrule(l){3-15} 
 &
   &
  \textbf{CLI↓} &
  18.2423 &
  18.2172 &
  18.4444 &
  17.2607 &
  18.5251 &
  17.9994 &
  17.5355 &
  14.7873 &
  17.5655 &
  18.4226 &
  18.2423 &
  17.6598 \\ \cmidrule(l){3-15} 
 &
   &
  \textbf{DCRS↓} &
  13.4198 &
  13.3969 &
  13.4585 &
  13.0143 &
  13.4976 &
  13.2664 &
  13.0972 &
  14.4874 &
  13.1547 &
  13.5274 &
  13.4198 &
  13.1462 \\ \cmidrule(l){3-15} 
 &
   &
  \textbf{FKGL↓} &
  17.3397 &
  17.4866 &
  17.7378 &
  16.7457 &
  17.8619 &
  17.0598 &
  16.5825 &
  14.3187 &
  16.6808 &
  17.4200 &
  17.3397 &
  16.6477 \\ \cmidrule(l){3-15} 
 &
   &
  \textbf{Rouge-L} &
  0.1419 &
  0.1377 &
  0.1411 &
  0.1307 &
  0.1378 &
  0.1350 &
  0.1301 &
  0.0715 &
  0.1412 &
  0.1475 &
  0.1419 &
  0.1285 \\ \cmidrule(l){2-15} 
 &
  \textbf{LexEval-QA} &
  \textbf{Rouge-L} &
  0.1471 &
  0.1494 &
  0.1468 &
  0.1466 &
  0.0375 &
  0.1508 &
  0.1434 &
  0.1398 &
  0.1442 &
  0.1450 &
  0.1533 &
  0.1533 \\ \cmidrule(l){2-15} 
 &
  \textbf{NFCats} &
  \textbf{Score} &
  4.4600 &
  4.5400 &
  4.3800 &
  3.8800 &
  4.0600 &
  4.2200 &
  3.7600 &
  3.8200 &
  4.3600 &
  4.3400 &
  4.4800 &
  4.5000 \\ \midrule
\multirow{12}{*}{\rotatebox{90}{\textbf{Short-Long}}} &
  \textbf{\begin{tabular}[c]{@{}c@{}}HelloBench\\ -Academic\\ \&Knowledge\\ -QA\end{tabular}} &
  \textbf{Avg. Score} &
  0.8479 &
  0.8354 &
  0.8508 &
  0.8616 &
  0.8155 &
  0.8475 &
  0.8743 &
  0.6752 &
  0.8524 &
  0.8679 &
  0.8700 &
  0.8700 \\ \cmidrule(l){2-15} 
 &
  \multirow{10}{*}{\textbf{JuDGE}} &
  \textbf{Amount Score} &
  0.2725 &
  0.2823 &
  0.3266 &
  0.3192 &
  0.3198 &
  0.3069 &
  0.3256 &
  0.0000 &
  0.2765 &
  0.3187 &
  0.3095 &
  0.3095 \\ \cmidrule(l){3-15} 
 &
   &
  \textbf{Crime Precision} &
  0.8500 &
  0.8500 &
  0.8400 &
  0.8800 &
  0.8500 &
  0.8800 &
  0.9000 &
  0.0000 &
  0.8900 &
  0.8800 &
  0.8300 &
  0.8300 \\ \cmidrule(l){3-15} 
 &
   &
  \textbf{Crime Recall} &
  0.8600 &
  0.8500 &
  0.8500 &
  0.9000 &
  0.8600 &
  0.9000 &
  0.9200 &
  0.0000 &
  0.9000 &
  0.8900 &
  0.8600 &
  0.8600 \\ \cmidrule(l){3-15} 
 &
   &
  \textbf{\begin{tabular}[c]{@{}c@{}}Judge\\ Bert Score\end{tabular}} &
  0.4613 &
  0.4909 &
  0.5412 &
  0.4561 &
  0.5276 &
  0.5823 &
  0.5357 &
  0.0000 &
  0.6531 &
  0.4135 &
  0.5941 &
  0.5941 \\ \cmidrule(l){3-15} 
 &
   &
  \textbf{Judge Meteor} &
  0.2767 &
  0.2529 &
  0.2748 &
  0.3023 &
  0.2633 &
  0.2626 &
  0.2819 &
  0.0000 &
  0.3128 &
  0.2836 &
  0.2994 &
  0.2994 \\ \cmidrule(l){3-15} 
 &
   &
  \textbf{\begin{tabular}[c]{@{}c@{}}Penalcode\\ Index Precision\end{tabular}} &
  0.6094 &
  0.6297 &
  0.6617 &
  0.5656 &
  0.5883 &
  0.6373 &
  0.5579 &
  0.0000 &
  0.5940 &
  0.4660 &
  0.6160 &
  0.6160 \\ \cmidrule(l){3-15} 
 &
   &
  \textbf{\begin{tabular}[c]{@{}c@{}}Penalcode\\ Index Recall\end{tabular}} &
  0.3502 &
  0.3355 &
  0.4085 &
  0.3742 &
  0.3025 &
  0.3140 &
  0.4000 &
  0.0000 &
  0.3678 &
  0.3894 &
  0.3752 &
  0.3752 \\ \cmidrule(l){3-15} 
 &
   &
  \textbf{\begin{tabular}[c]{@{}c@{}}Reasoning\\ Bert Score\end{tabular}} &
  0.6964 &
  0.6930 &
  0.7142 &
  0.7110 &
  0.6901 &
  0.7289 &
  0.7180 &
  0.4035 &
  0.7324 &
  0.6979 &
  0.7269 &
  0.7269 \\ \cmidrule(l){3-15} 
 &
   &
  \textbf{\begin{tabular}[c]{@{}c@{}}Reasoning\\ Meteor\end{tabular}} &
  0.3206 &
  0.3493 &
  0.3735 &
  0.3552 &
  0.3422 &
  0.3532 &
  0.3576 &
  0.0000 &
  0.3246 &
  0.3840 &
  0.3577 &
  0.3577 \\ \cmidrule(l){3-15} 
 &
   &
  \textbf{Time Score} &
  0.4384 &
  0.5031 &
  0.5436 &
  0.5142 &
  0.5296 &
  0.5209 &
  0.4900 &
  0.0000 &
  0.5177 &
  0.4729 &
  0.5333 &
  0.5333 \\ \cmidrule(l){2-15} 
 &
  \textbf{WritingPrompts} &
  \textbf{Meteor} &
  0.2254 &
  0.2150 &
  0.2267 &
  0.2188 &
  0.2192 &
  0.2159 &
  0.2200 &
  0.0535 &
  0.2157 &
  0.2111 &
  0.2254 &
  0.2342 \\ \midrule
\multirow{6}{*}{\rotatebox{90}{\textbf{Long-Long}}} &
  \textbf{\begin{tabular}[c]{@{}c@{}}HelloBench\\ -Academic\\ \&Knowledge\\ -Writing\end{tabular}} &
  \textbf{Avg. Score} &
  0.8186 &
  0.8289 &
  0.7760 &
  0.8642 &
  0.8000 &
  0.8431 &
  0.7799 &
  0.3147 &
  0.8196 &
  0.8314 &
  0.8191 &
  0.8260 \\ \cmidrule(l){2-15} 
 &
  \textbf{\begin{tabular}[c]{@{}c@{}}HelloBench\\ -Creative\&Design\end{tabular}} &
  \textbf{Avg. Score} &
  0.8027 &
  0.8016 &
  0.8289 &
  0.7767 &
  0.7709 &
  0.8030 &
  0.7941 &
  0.4096 &
  0.7576 &
  0.8129 &
  0.7956 &
  0.8278 \\ \cmidrule(l){2-15} 
 &
  \textbf{LexEval-Judge} &
  \textbf{Rouge-L} &
  0.1629 &
  0.1601 &
  0.1582 &
  0.1402 &
  0.0834 &
  0.1391 &
  0.1396 &
  0.0725 &
  0.1474 &
  0.1855 &
  0.1629 &
  0.1626 \\ \cmidrule(l){2-15} 
 &
  \textbf{\begin{tabular}[c]{@{}c@{}}WritingBench\\ -Academic\\ \&Engineering\end{tabular}} &
  \textbf{Score} &
  6.0588 &
  5.5294 &
  5.4412 &
  5.5294 &
  5.0000 &
  5.0294 &
  5.2353 &
  3.6765 &
  5.3824 &
  5.2647 &
  6.1471 &
  6.1471 \\ \cmidrule(l){2-15} 
 &
  \textbf{\begin{tabular}[c]{@{}c@{}}WritingBench\\ -Creative\&Design\end{tabular}} &
  \textbf{Score} &
  5.5581 &
  5.4302 &
  5.3140 &
  5.2209 &
  5.0581 &
  5.3372 &
  5.2209 &
  3.6163 &
  5.2326 &
  5.3488 &
  5.6163 &
  5.6744 \\ \cmidrule(l){2-15} 
 &
  \textbf{\begin{tabular}[c]{@{}c@{}}WritingBench\\ -Politics\&Law\end{tabular}} &
  \textbf{Score} &
  5.1463 &
  4.4146 &
  4.5854 &
  4.8293 &
  4.2927 &
  4.7073 &
  4.8780 &
  3.6585 &
  4.7073 &
  4.7561 &
  5.0488 &
  4.9512 \\ \midrule
\multirow{7}{*}{\rotatebox{90}{\textbf{Long-Short}}} &
  \multirow{4}{*}{\textbf{IdeaBench}} &
  \textbf{Bert Score} &
  0.5640 &
  0.5628 &
  0.5654 &
  0.5476 &
  0.5621 &
  0.5496 &
  - &
  0.5314 &
  0.5620 &
  0.5685 &
  0.5640 &
  0.5660 \\ \cmidrule(l){3-15} 
 &
   &
  \textbf{\begin{tabular}[c]{@{}c@{}}LLM Feasibility\\ Ranking Score\end{tabular}} &
  0.1200 &
  0.1400 &
  0.0800 &
  0.0000 &
  0.1400 &
  0.0733 &
  - &
  0.0000 &
  0.0267 &
  0.0667 &
  0.1200 &
  0.0967 \\ \cmidrule(l){3-15} 
 &
   &
  \textbf{\begin{tabular}[c]{@{}c@{}}LLM Novelty\\ Ranking Score\end{tabular}} &
  0.6667 &
  0.6600 &
  0.5733 &
  0.0467 &
  0.6400 &
  0.1133 &
  - &
  0.0067 &
  0.2067 &
  0.3467 &
  0.6667 &
  0.6800 \\ \cmidrule(l){3-15} 
 &
   &
  \textbf{\begin{tabular}[c]{@{}c@{}}LLM \\ Rating Score\end{tabular}} &
  4.5200 &
  4.3400 &
  4.6200 &
  4.3600 &
  4.7200 &
  4.4200 &
  - &
  4.4000 &
  4.5600 &
  4.7200 &
  4.5200 &
  4.5400 \\ \cmidrule(l){2-15} 
 &
  \textbf{\begin{tabular}[c]{@{}c@{}}LexEval\\ -Summarization\end{tabular}} &
  \textbf{Rouge-L} &
  0.2212 &
  0.2139 &
  0.2124 &
  0.2277 &
  0.1304 &
  0.2211 &
  - &
  0.1612 &
  0.2164 &
  0.2185 &
  0.2307 &
  0.2307 \\ \cmidrule(l){2-15} 
 &
  \multirow{2}{*}{\textbf{LimitGen-Syn}} &
  \textbf{Accuracy} &
  0.5000 &
  0.3400 &
  0.4200 &
  0.5000 &
  0.4200 &
  0.3800 &
  - &
  0.0000 &
  0.0600 &
  0.2600 &
  0.5000 &
  0.6600 \\ \cmidrule(l){3-15} 
 &
   &
  \textbf{Rating} &
  1.3000 &
  0.8400 &
  1.1000 &
  1.3400 &
  1.1000 &
  0.9600 &
  - &
  0.0000 &
  0.1400 &
  0.7000 &
  1.3000 &
  1.9000 \\ \bottomrule
\end{tabular}%
}
\label{tab:detailphi4}
\end{table*}

%% file: acl_latex.bbl
\begin{thebibliography}{47}
\providecommand{\natexlab}[1]{#1}

\bibitem[{Abdin et~al.(2024)Abdin, Aneja, Behl, Bubeck, Eldan, Gunasekar,
  Harrison, Hewett, Javaheripi, Kauffmann, Lee, Lee, Li, Liu, Mendes, Nguyen,
  Price, de~Rosa, Saarikivi, Salim, Shah, Wang, Ward, Wu, Yu, Zhang, and
  Zhang}]{abdin2024phi4technicalreport}
Marah Abdin, Jyoti Aneja, Harkirat Behl, Sébastien Bubeck, Ronen Eldan, Suriya
  Gunasekar, Michael Harrison, Russell~J. Hewett, Mojan Javaheripi, Piero
  Kauffmann, James~R. Lee, Yin~Tat Lee, Yuanzhi Li, Weishung Liu, Caio C.~T.
  Mendes, Anh Nguyen, Eric Price, Gustavo de~Rosa, Olli Saarikivi, and 8
  others. 2024.
\newblock \href {https://arxiv.org/abs/2412.08905} {Phi-4 technical report}.
\newblock \emph{Preprint}, arXiv:2412.08905.

\bibitem[{Ai et~al.(2025)Ai, Tang, Wang, Long, Su, and
  Liu}]{ai2025memorybenchbenchmarkmemorycontinual}
Qingyao Ai, Yichen Tang, Changyue Wang, Jianming Long, Weihang Su, and Yiqun
  Liu. 2025.
\newblock \href {https://arxiv.org/abs/2510.17281} {Memorybench: A benchmark
  for memory and continual learning in llm systems}.
\newblock \emph{Preprint}, arXiv:2510.17281.

\bibitem[{Chhikara et~al.(2025)Chhikara, Khant, Aryan, Singh, and Yadav}]{Mem0}
Prateek Chhikara, Dev Khant, Saket Aryan, Taranjeet Singh, and Deshraj Yadav.
  2025.
\newblock Mem0: Building production-ready ai agents with scalable long-term
  memory.
\newblock \emph{arXiv preprint arXiv:2504.19413}.

\bibitem[{Croft et~al.(2009)Croft, Metzler, and Strohman}]{croft2010search}
W.~Bruce Croft, Donald Metzler, and Trevor Strohman. 2009.
\newblock \href {http://www.search-engines-book.com/} {\emph{Search Engines -
  Information Retrieval in Practice}}.
\newblock Pearson Education.

\bibitem[{Dewey(2012)}]{dewey2012experience}
John Dewey. 2012.
\newblock \emph{Experience and nature}.
\newblock Courier Corporation.

\bibitem[{Douze et~al.(2024)Douze, Guzhva, Deng, Johnson, Szilvasy, Mazaré,
  Lomeli, Hosseini, and Jégou}]{douze2024faiss}
Matthijs Douze, Alexandr Guzhva, Chengqi Deng, Jeff Johnson, Gergely Szilvasy,
  Pierre-Emmanuel Mazaré, Maria Lomeli, Lucas Hosseini, and Hervé Jégou.
  2024.
\newblock \href {https://arxiv.org/abs/2401.08281} {The faiss library}.

\bibitem[{Fang et~al.(2024)Fang, Zhan, Ai, Mao, Su, Chen, and
  Liu}]{fang2024scaling}
Yan Fang, Jingtao Zhan, Qingyao Ai, Jiaxin Mao, Weihang Su, Jia Chen, and Yiqun
  Liu. 2024.
\newblock Scaling laws for dense retrieval.
\newblock In \emph{Proceedings of the 47th International ACM SIGIR Conference
  on Research and Development in Information Retrieval}, pages 1339--1349.

\bibitem[{Feng et~al.(2025)Feng, Wang, Zhou, Wang, Zhan, Li, Li, and
  Zhu}]{feng2025evoagentselfevolvingagentcontinual}
Tongtong Feng, Xin Wang, Zekai Zhou, Ren Wang, Yuwei Zhan, Guangyao Li, Qing
  Li, and Wenwu Zhu. 2025.
\newblock \href {https://arxiv.org/abs/2502.05907} {Evoagent: Self-evolving
  agent with continual world model for long-horizon tasks}.
\newblock \emph{Preprint}, arXiv:2502.05907.

\bibitem[{{Gao} et~al.(2026){Gao}, Geng, Hua, Hu, Juan, Liu, Liu, Qiu, Qi, Wu,
  Wang, Xiao, Zhou, Zhang, Zhang, Xiang, Fang, Zhao, Liu, Ren, Qian, Wang, Hu,
  Wang, Wu, Ji, and Wang}]{gao2025surveyselfevolvingagentspath}
{Huan-ang} {Gao}, Jiayi Geng, Wenyue Hua, Mengkang Hu, Xinzhe Juan, Hongzhang
  Liu, Shilong Liu, Jiahao Qiu, Xuan Qi, Yiran Wu, Hongru Wang, Han Xiao,
  Yuhang Zhou, Shaokun Zhang, Jiayi Zhang, Jinyu Xiang, Yixiong Fang, Qiwen
  Zhao, Dongrui Liu, and 8 others. 2026.
\newblock \href {https://arxiv.org/abs/2507.21046} {A survey of self-evolving
  agents: What, when, how, and where to evolve on the path to artificial super
  intelligence}.
\newblock \emph{Preprint}, arXiv:2507.21046.

\bibitem[{Han et~al.(2025)Han, Chen, Sankararaman, Peng, Xu, Helenowski, Peng,
  Kumar, Wang, Fang, and Talebzadeh}]{han2025reinforcementlearninguserfeedback}
Eric Han, Jun Chen, Karthik~Abinav Sankararaman, Xiaoliang Peng, Tengyu Xu,
  Eryk Helenowski, Kaiyan Peng, Mrinal Kumar, Sinong Wang, Han Fang, and Arya
  Talebzadeh. 2025.
\newblock \href {https://arxiv.org/abs/2505.14946} {Reinforcement learning from
  user feedback}.
\newblock \emph{Preprint}, arXiv:2505.14946.

\bibitem[{Hu et~al.(2022)Hu, Shen, Wallis, Allen{-}Zhu, Li, Wang, Wang, and
  Chen}]{lora}
Edward~J. Hu, Yelong Shen, Phillip Wallis, Zeyuan Allen{-}Zhu, Yuanzhi Li,
  Shean Wang, Lu~Wang, and Weizhu Chen. 2022.
\newblock \href {https://openreview.net/forum?id=nZeVKeeFYf9} {Lora: Low-rank
  adaptation of large language models}.
\newblock In \emph{The Tenth International Conference on Learning
  Representations, {ICLR} 2022, Virtual Event, April 25-29, 2022}.
  OpenReview.net.

\bibitem[{Joachims(2002)}]{searchengine}
Thorsten Joachims. 2002.
\newblock \href {https://doi.org/10.1145/775047.775067} {Optimizing search
  engines using clickthrough data}.
\newblock In \emph{Proceedings of the Eighth ACM SIGKDD International
  Conference on Knowledge Discovery and Data Mining}, KDD '02, page 133–142,
  New York, NY, USA. Association for Computing Machinery.

\bibitem[{Kang et~al.(2025)Kang, Ji, Zhao, and Bai}]{kang-etal-2025-memory}
Jiazheng Kang, Mingming Ji, Zhe Zhao, and Ting Bai. 2025.
\newblock \href {https://doi.org/10.18653/v1/2025.emnlp-main.1318} {Memory {OS}
  of {AI} agent}.
\newblock In \emph{Proceedings of the 2025 Conference on Empirical Methods in
  Natural Language Processing}, pages 25961--25970, Suzhou, China. Association
  for Computational Linguistics.

\bibitem[{Kaplan et~al.(2020)Kaplan, McCandlish, Henighan, Brown, Chess, Child,
  Gray, Radford, Wu, and Amodei}]{kaplan2020scalinglawsneurallanguage}
Jared Kaplan, Sam McCandlish, Tom Henighan, Tom~B. Brown, Benjamin Chess, Rewon
  Child, Scott Gray, Alec Radford, Jeffrey Wu, and Dario Amodei. 2020.
\newblock \href {https://arxiv.org/abs/2001.08361} {Scaling laws for neural
  language models}.
\newblock \emph{Preprint}, arXiv:2001.08361.

\bibitem[{Kelly and Teevan(2003)}]{userfeedback}
Diane Kelly and Jaime Teevan. 2003.
\newblock \href {https://doi.org/10.1145/959258.959260} {Implicit feedback for
  inferring user preference: a bibliography}.
\newblock \emph{SIGIR Forum}, 37(2):18–28.

\bibitem[{Kwon et~al.(2023)Kwon, Li, Zhuang, Sheng, Zheng, Yu, Gonzalez, Zhang,
  and Stoica}]{kwon2023efficient}
Woosuk Kwon, Zhuohan Li, Siyuan Zhuang, Ying Sheng, Lianmin Zheng, Cody~Hao Yu,
  Joseph~E. Gonzalez, Hao Zhang, and Ion Stoica. 2023.
\newblock Efficient memory management for large language model serving with
  pagedattention.
\newblock In \emph{Proceedings of the ACM SIGOPS 29th Symposium on Operating
  Systems Principles}.

\bibitem[{Levine et~al.(2020)Levine, Kumar, Tucker, and
  Fu}]{levine2020offlinereinforcementlearningtutorial}
Sergey Levine, Aviral Kumar, George Tucker, and Justin Fu. 2020.
\newblock \href {https://arxiv.org/abs/2005.01643} {Offline reinforcement
  learning: Tutorial, review, and perspectives on open problems}.
\newblock \emph{Preprint}, arXiv:2005.01643.

\bibitem[{Lewis et~al.(2020)Lewis, Perez, Piktus, Petroni, Karpukhin, Goyal,
  K\"{u}ttler, Lewis, Yih, Rockt\"{a}schel, Riedel, and
  Kiela}]{lewis2020retrieval}
Patrick Lewis, Ethan Perez, Aleksandra Piktus, Fabio Petroni, Vladimir
  Karpukhin, Naman Goyal, Heinrich K\"{u}ttler, Mike Lewis, Wen-tau Yih, Tim
  Rockt\"{a}schel, Sebastian Riedel, and Douwe Kiela. 2020.
\newblock Retrieval-augmented generation for knowledge-intensive nlp tasks.
\newblock \emph{Advances in Neural Information Processing Systems},
  33:9459--9474.

\bibitem[{Li et~al.(2025{\natexlab{a}})Li, Jiang, Huang, Beigi, Zhao, Tan,
  Bhattacharjee, Jiang, Chen, Wu, Shu, Cheng, and
  Liu}]{li2025generationjudgmentopportunitieschallenges}
Dawei Li, Bohan Jiang, Liangjie Huang, Alimohammad Beigi, Chengshuai Zhao, Zhen
  Tan, Amrita Bhattacharjee, Yuxuan Jiang, Canyu Chen, Tianhao Wu, Kai Shu,
  Lu~Cheng, and Huan Liu. 2025{\natexlab{a}}.
\newblock \href {https://arxiv.org/abs/2411.16594} {From generation to
  judgment: Opportunities and challenges of llm-as-a-judge}.
\newblock \emph{Preprint}, arXiv:2411.16594.

\bibitem[{Li et~al.(2025{\natexlab{b}})Li, Xi, Li, Chen, Chen, Song, Niu, Wang,
  Yang, Tang, Yu, Zhao, Wang, Liu, Lin, Wang, Huo, Chen, Chen, Li, Tao, Lai,
  Wu, Tang, Wang, Fan, Zhang, Zhang, Yan, Yang, Xu, Xu, Chen, Wang, Yang,
  Zhang, Xu, Chen, and Xiong}]{li2025memosmemoryosai}
Zhiyu Li, Chenyang Xi, Chunyu Li, Ding Chen, Boyu Chen, Shichao Song, Simin
  Niu, Hanyu Wang, Jiawei Yang, Chen Tang, Qingchen Yu, Jihao Zhao, Yezhaohui
  Wang, Peng Liu, Zehao Lin, Pengyuan Wang, Jiahao Huo, Tianyi Chen, Kai Chen,
  and 20 others. 2025{\natexlab{b}}.
\newblock \href {https://arxiv.org/abs/2507.03724} {Memos: A memory os for ai
  system}.
\newblock \emph{Preprint}, arXiv:2507.03724.

\bibitem[{Maharana et~al.(2024)Maharana, Lee, Tulyakov, Bansal, Barbieri, and
  Fang}]{maharana-etal-2024-evaluating}
Adyasha Maharana, Dong-Ho Lee, Sergey Tulyakov, Mohit Bansal, Francesco
  Barbieri, and Yuwei Fang. 2024.
\newblock \href {https://doi.org/10.18653/v1/2024.acl-long.747} {Evaluating
  very long-term conversational memory of {LLM} agents}.
\newblock In \emph{Proceedings of the 62nd Annual Meeting of the Association
  for Computational Linguistics (Volume 1: Long Papers)}, pages 13851--13870,
  Bangkok, Thailand. Association for Computational Linguistics.

\bibitem[{Müllner(2011)}]{müllner2011modernhierarchicalagglomerativeclustering}
Daniel Müllner. 2011.
\newblock \href {https://arxiv.org/abs/1109.2378} {Modern hierarchical,
  agglomerative clustering algorithms}.
\newblock \emph{Preprint}, arXiv:1109.2378.

\bibitem[{Northcutt et~al.(2022)Northcutt, Jiang, and
  Chuang}]{northcutt2022confidentlearningestimatinguncertainty}
Curtis~G. Northcutt, Lu~Jiang, and Isaac~L. Chuang. 2022.
\newblock \href {https://arxiv.org/abs/1911.00068} {Confident learning:
  Estimating uncertainty in dataset labels}.
\newblock \emph{Preprint}, arXiv:1911.00068.

\bibitem[{Novikov et~al.(2025)Novikov, Vũ, Eisenberger, Dupont, Huang, Wagner,
  Shirobokov, Kozlovskii, Ruiz, Mehrabian, Kumar, See, Chaudhuri, Holland,
  Davies, Nowozin, Kohli, and
  Balog}]{novikov2025alphaevolvecodingagentscientific}
Alexander Novikov, Ngân Vũ, Marvin Eisenberger, Emilien Dupont, Po-Sen Huang,
  Adam~Zsolt Wagner, Sergey Shirobokov, Borislav Kozlovskii, Francisco J.~R.
  Ruiz, Abbas Mehrabian, M.~Pawan Kumar, Abigail See, Swarat Chaudhuri, George
  Holland, Alex Davies, Sebastian Nowozin, Pushmeet Kohli, and Matej Balog.
  2025.
\newblock \href {https://arxiv.org/abs/2506.13131} {Alphaevolve: A coding agent
  for scientific and algorithmic discovery}.
\newblock \emph{Preprint}, arXiv:2506.13131.

\bibitem[{Papineni et~al.(2002)Papineni, Roukos, Ward, and
  Zhu}]{papineni-etal-2002-bleu}
Kishore Papineni, Salim Roukos, Todd Ward, and Wei-Jing Zhu. 2002.
\newblock \href {https://doi.org/10.3115/1073083.1073135} {{B}leu: a method for
  automatic evaluation of machine translation}.
\newblock In \emph{Proceedings of the 40th Annual Meeting of the Association
  for Computational Linguistics}, pages 311--318, Philadelphia, Pennsylvania,
  USA. Association for Computational Linguistics.

\bibitem[{Pedregosa et~al.(2011)Pedregosa, Varoquaux, Gramfort, Michel,
  Thirion, Grisel, Blondel, Prettenhofer, Weiss, Dubourg, Vanderplas, Passos,
  Cournapeau, Brucher, Perrot, and Duchesnay}]{scikit-learn}
F.~Pedregosa, G.~Varoquaux, A.~Gramfort, V.~Michel, B.~Thirion, O.~Grisel,
  M.~Blondel, P.~Prettenhofer, R.~Weiss, V.~Dubourg, J.~Vanderplas, A.~Passos,
  D.~Cournapeau, M.~Brucher, M.~Perrot, and E.~Duchesnay. 2011.
\newblock Scikit-learn: Machine learning in {P}ython.
\newblock \emph{Journal of Machine Learning Research}, 12:2825--2830.

\bibitem[{Peng et~al.(2026)Peng, Qi, Wang, Yao, Hou, and
  Li}]{peng2026wildrewardlearningrewardmodels}
Hao Peng, Yunjia Qi, Xiaozhi Wang, Zijun Yao, Lei Hou, and Juanzi Li. 2026.
\newblock \href {https://arxiv.org/abs/2602.08829} {Wildreward: Learning reward
  models from in-the-wild human interactions}.
\newblock \emph{Preprint}, arXiv:2602.08829.

\bibitem[{Rafailov et~al.(2023)Rafailov, Sharma, Mitchell, Manning, Ermon, and
  Finn}]{dpo}
Rafael Rafailov, Archit Sharma, Eric Mitchell, Christopher~D. Manning, Stefano
  Ermon, and Chelsea Finn. 2023.
\newblock Direct preference optimization: Your language model is secretly a
  reward model.
\newblock In \emph{Advances in Neural Information Processing Systems 36: Annual
  Conference on Neural Information Processing Systems 2023, NeurIPS 2023, New
  Orleans, LA, USA, December 10 - 16, 2023}.

\bibitem[{Robertson and Zaragoza(2009)}]{robertson2009probabilistic}
Stephen Robertson and Hugo Zaragoza. 2009.
\newblock \href {https://doi.org/10.1561/1500000019} {The probabilistic
  relevance framework: Bm25 and beyond}.
\newblock \emph{Found. Trends Inf. Retr.}, 3(4):333–389.

\bibitem[{Shi et~al.(2025{\natexlab{a}})Shi, Xu, Wang, Qin, Wang, Wang, Wang,
  Ebrahimi, and Wang}]{10.1145/3735633}
Haizhou Shi, Zihao Xu, Hengyi Wang, Weiyi Qin, Wenyuan Wang, Yibin Wang, Zifeng
  Wang, Sayna Ebrahimi, and Hao Wang. 2025{\natexlab{a}}.
\newblock \href {https://doi.org/10.1145/3735633} {Continual learning of large
  language models: A comprehensive survey}.
\newblock \emph{ACM Comput. Surv.}, 58(5).

\bibitem[{Shi et~al.(2025{\natexlab{b}})Shi, Xu, Wang, Qin, Wang, Wang, Wang,
  Ebrahimi, and Wang}]{shi2025continuallearning}
Haizhou Shi, Zihao Xu, Hengyi Wang, Weiyi Qin, Wenyuan Wang, Yibin Wang, Zifeng
  Wang, Sayna Ebrahimi, and Hao Wang. 2025{\natexlab{b}}.
\newblock \href {https://doi.org/10.1145/3735633} {Continual learning of large
  language models: A comprehensive survey}.
\newblock \emph{ACM Comput. Surv.}
\newblock Just Accepted.

\bibitem[{Shi et~al.(2026)Shi, Wang, Yang, Lin, He, Wan, Zhou, Jauhar, Chen,
  Xia, Zhang, Zhao, Xu, Song, and Neville}]{wildfeedback}
Taiwei Shi, Zhuoer Wang, Longqi Yang, Ying-Chun Lin, Zexue He, Mengting Wan,
  Pei Zhou, Sujay Jauhar, Sihao Chen, Shan Xia, Hongfei Zhang, Jieyu Zhao,
  Xiaofeng Xu, Xia Song, and Jennifer Neville. 2026.
\newblock \href {https://arxiv.org/abs/2408.15549} {Wildfeedback: Aligning llms
  with in-situ user interactions and feedback}.
\newblock \emph{Preprint}, arXiv:2408.15549.

\bibitem[{Student(1908)}]{student1908probable}
Student. 1908.
\newblock The probable error of a mean.
\newblock \emph{Biometrika}, pages 1--25.

\bibitem[{Swayamdipta et~al.(2020)Swayamdipta, Schwartz, Lourie, Wang,
  Hajishirzi, Smith, and Choi}]{swayamdipta-etal-2020-dataset}
Swabha Swayamdipta, Roy Schwartz, Nicholas Lourie, Yizhong Wang, Hannaneh
  Hajishirzi, Noah~A. Smith, and Yejin Choi. 2020.
\newblock \href {https://doi.org/10.18653/v1/2020.emnlp-main.746} {Dataset
  cartography: Mapping and diagnosing datasets with training dynamics}.
\newblock In \emph{Proceedings of the 2020 Conference on Empirical Methods in
  Natural Language Processing (EMNLP)}, pages 9275--9293, Online. Association
  for Computational Linguistics.

\bibitem[{Villalobos et~al.(2024)Villalobos, Ho, Sevilla, Besiroglu, Heim, and
  Hobbhahn}]{10.5555/3692070.3694094}
Pablo Villalobos, Anson Ho, Jaime Sevilla, Tamay Besiroglu, Lennart Heim, and
  Marius Hobbhahn. 2024.
\newblock Position: will we run out of data? limits of llm scaling based on
  human-generated data.
\newblock In \emph{Proceedings of the 41st International Conference on Machine
  Learning}, ICML'24. JMLR.org.

\bibitem[{von Werra et~al.(2020)von Werra, Belkada, Tunstall, Beeching, Thrush,
  Lambert, Huang, Rasul, and Gallouédec}]{vonwerra2022trl}
Leandro von Werra, Younes Belkada, Lewis Tunstall, Edward Beeching, Tristan
  Thrush, Nathan Lambert, Shengyi Huang, Kashif Rasul, and Quentin Gallouédec.
  2020.
\newblock Trl: Transformer reinforcement learning.
\newblock \url{https://github.com/huggingface/trl}.

\bibitem[{Wei et~al.(2026)Wei, Sachdeva, Coleman, He, Bei, Ning, Ai, Li, He,
  Chi, Wang, Chen, Pereira, Kang, and
  Cheng}]{wei2025evomemorybenchmarkingllmagent}
Tianxin Wei, Noveen Sachdeva, Benjamin Coleman, Zhankui He, Yuanchen Bei,
  Xuying Ning, Mengting Ai, Yunzhe Li, Jingrui He, Ed~H. Chi, Chi Wang, Shuo
  Chen, Fernando Pereira, Wang-Cheng Kang, and Derek~Zhiyuan Cheng. 2026.
\newblock \href {https://arxiv.org/abs/2511.20857} {Evo-memory: Benchmarking
  llm agent test-time learning with self-evolving memory}.
\newblock \emph{Preprint}, arXiv:2511.20857.

\bibitem[{Xu et~al.(2025)Xu, Mei, Gao, Tan, Liang, and Zhang}]{xu2025amem}
Wujiang Xu, Kai Mei, Hang Gao, Juntao Tan, Zujie Liang, and Yongfeng Zhang.
  2025.
\newblock A-mem: Agentic memory for llm agents.
\newblock \emph{arXiv preprint arXiv:2502.12110}.

\bibitem[{Yang et~al.(2025)Yang, Li, Yang, Zhang, Hui, Zheng, Yu, Gao, Huang,
  Lv, Zheng, Liu, Zhou, Huang, Hu, Ge, Wei, Lin, Tang, Yang, Tu, Zhang, Yang,
  Yang, Zhou, Zhou, Lin, Dang, Bao, Yang, Yu, Deng, Li, Xue, Li, Zhang, Wang,
  Zhu, Men, Gao, Liu, Luo, Li, Tang, Yin, Ren, Wang, Zhang, Ren, Fan, Su,
  Zhang, Zhang, Wan, Liu, Wang, Cui, Zhang, Zhou, and
  Qiu}]{yang2025qwen3technicalreport}
An~Yang, Anfeng Li, Baosong Yang, Beichen Zhang, Binyuan Hui, Bo~Zheng, Bowen
  Yu, Chang Gao, Chengen Huang, Chenxu Lv, Chujie Zheng, Dayiheng Liu, Fan
  Zhou, Fei Huang, Feng Hu, Hao Ge, Haoran Wei, Huan Lin, Jialong Tang, and 41
  others. 2025.
\newblock \href {https://arxiv.org/abs/2505.09388} {Qwen3 technical report}.
\newblock \emph{Preprint}, arXiv:2505.09388.

\bibitem[{Yin et~al.(2025)Yin, Wang, Bao, Xu, and
  Wang}]{yin2025clickspreferencemultistagealignment}
Junhao Yin, Haolin Wang, Peng Bao, Ju~Xu, and Yongliang Wang. 2025.
\newblock \href {https://arxiv.org/abs/2508.15811} {From clicks to preference:
  A multi-stage alignment framework for generative query suggestion in
  conversational system}.
\newblock \emph{Preprint}, arXiv:2508.15811.

\bibitem[{Zhai et~al.(2025)Zhai, Tao, Chen, Zou, Chen, Fu, Mai, Yu, Deng, Cao,
  Liu, Ding, and Zhou}]{zhai2025agentevolverefficientselfevolvingagent}
Yunpeng Zhai, Shuchang Tao, Cheng Chen, Anni Zou, Ziqian Chen, Qingxu Fu,
  Shinji Mai, Li~Yu, Jiaji Deng, Zouying Cao, Zhaoyang Liu, Bolin Ding, and
  Jingren Zhou. 2025.
\newblock \href {https://arxiv.org/abs/2511.10395} {Agentevolver: Towards
  efficient self-evolving agent system}.
\newblock \emph{Preprint}, arXiv:2511.10395.

\bibitem[{Zhang et~al.(2026)Zhang, Wang, Zhou, Liao, Feng, Zhang, Wen, Li,
  Xiong, Qi, Tang, and Wen}]{zhang2026memrlselfevolvingagentsruntime}
Shengtao Zhang, Jiaqian Wang, Ruiwen Zhou, Junwei Liao, Yuchen Feng, Weinan
  Zhang, Ying Wen, Zhiyu Li, Feiyu Xiong, Yutao Qi, Bo~Tang, and Muning Wen.
  2026.
\newblock \href {https://arxiv.org/abs/2601.03192} {Memrl: Self-evolving agents
  via runtime reinforcement learning on episodic memory}.
\newblock \emph{Preprint}, arXiv:2601.03192.

\bibitem[{Zhang et~al.(2019)Zhang, Yao, Sun, and Tay}]{recommender}
Shuai Zhang, Lina Yao, Aixin Sun, and Yi~Tay. 2019.
\newblock \href {https://doi.org/10.1145/3285029} {Deep learning based
  recommender system: A survey and new perspectives}.
\newblock \emph{ACM Comput. Surv.}, 52(1).

\bibitem[{Zhang et~al.(2025)Zhang, Li, Long, Zhang, Lin, Yang, Xie, Yang, Liu,
  Lin, Huang, and Zhou}]{zhang2025qwen3embeddingadvancingtext}
Yanzhao Zhang, Mingxin Li, Dingkun Long, Xin Zhang, Huan Lin, Baosong Yang,
  Pengjun Xie, An~Yang, Dayiheng Liu, Junyang Lin, Fei Huang, and Jingren Zhou.
  2025.
\newblock \href {https://arxiv.org/abs/2506.05176} {Qwen3 embedding: Advancing
  text embedding and reranking through foundation models}.
\newblock \emph{Preprint}, arXiv:2506.05176.

\bibitem[{Zhao et~al.(2026)Zhao, Xie, Liu, Huang, Pang, Chen, and
  Grover}]{zhao2026selfdistilledreasoneronpolicyselfdistillation}
Siyan Zhao, Zhihui Xie, Mengchen Liu, Jing Huang, Guan Pang, Feiyu Chen, and
  Aditya Grover. 2026.
\newblock \href {https://arxiv.org/abs/2601.18734} {Self-distilled reasoner:
  On-policy self-distillation for large language models}.
\newblock \emph{Preprint}, arXiv:2601.18734.

\bibitem[{Zhao et~al.(2024)Zhao, Ren, Hessel, Cardie, Choi, and
  Deng}]{wildchat}
Wenting Zhao, Xiang Ren, Jack Hessel, Claire Cardie, Yejin Choi, and Yuntian
  Deng. 2024.
\newblock \href {https://openreview.net/forum?id=Bl8u7ZRlbM} {Wildchat: 1m
  chatgpt interaction logs in the wild}.
\newblock In \emph{The Twelfth International Conference on Learning
  Representations, {ICLR} 2024, Vienna, Austria, May 7-11, 2024}.
  OpenReview.net.

\bibitem[{Zheng et~al.(2025)Zheng, Qiu, Shi, and Ma}]{lifelong}
Junhao Zheng, Shengjie Qiu, Chengming Shi, and Qianli Ma. 2025.
\newblock \href {https://doi.org/10.1145/3716629} {Towards lifelong learning of
  large language models: A survey}.
\newblock \emph{ACM Comput. Surv.}, 57(8).

\end{thebibliography}
